\documentclass[final]{siamltex}

\usepackage{amssymb}
\usepackage{amsmath}
\usepackage{graphicx}
\usepackage{subfig}
\usepackage{algorithm}
\usepackage{algpseudocode}
\usepackage{psfrag}
\usepackage{url}
\usepackage{array}
\usepackage[sort]{natbib}

\usepackage[usenames]{color}
\usepackage[normalem]{ulem}

\newcolumntype{S}{>{\centering\arraybackslash} m{0.8cm}}
\newcolumntype{F}{>{\centering\arraybackslash} m{.2\linewidth}}

\providecommand{\tabularnewline}{\\}
\begin{document}

\title{Orientation
Determination of Cryo-EM images Using Least Unsquared Deviations}

\author{Lanhui Wang%
\thanks{The Program in Applied and Computational Mathematics (PACM), Princeton University, Fine Hall, Washington Road, Princeton, NJ 08544-1000, USA, \texttt{lanhuiw@math.princeton.edu}, Corresponding author. Tel.: +1 609 258 5785; fax: +1 609 258 1735. }
\and Amit Singer%
\thanks{Department of Mathematics and PACM, Princeton University, Fine Hall, Washington Road, Princeton, NJ 08544-1000, USA, \texttt{amits@math.princeton.edu} }
\and Zaiwen Wen%
\thanks{Department of Mathematics, MOE-LSC  and Institute of Natural Sciences,
Shanghai Jiao Tong University, Pao Yue-Kong Library, 800 Dongchuan Rd,
Shanghai, China, \texttt{zw2109@sjtu.edu.cn} } }

\date{}
\maketitle

\begin{abstract}
A major challenge in single particle reconstruction from cryo-electron microscopy is to
establish a reliable ab-initio three-dimensional model using two-dimensional projection images with
unknown orientations. Common-lines based methods
estimate the orientations without additional geometric
information. However, such methods fail when the detection rate of common-lines
is too low due to the high level of noise in the images. An approximation to the
least squares global self consistency error was obtained in \cite{Amit_eig_sdp}
using convex relaxation by semidefinite programming. In this paper we introduce
a more robust global self consistency error and show that the corresponding
optimization problem can be solved via semidefinite relaxation. In order to
prevent artificial clustering of the estimated viewing directions, we further
introduce a spectral norm term that is added as a constraint or as a
regularization term to the relaxed minimization problem. The resulted problems are solved by using either the alternating direction method of multipliers
or an iteratively reweighted least squares procedure.  
Numerical
  experiments with both simulated and real images demonstrate that the proposed methods significantly reduce the orientation
estimation error when the detection rate of common-lines is low.
\end{abstract}


\begin{AM}
92E10, 68U10, 94A08, 92C55, 90C22, 90C25
\end{AM}
\section{Introduction}

In single particle analysis, cryo-electron microscopy (Cryo-EM) is
used to attain a resolution sufficient to interpret fine details in three-dimensional
(3D) macromolecular structures \cite{Frank1996, vanheel_cryo-em,ribo_cryo-em,Fred_resolution}. Cryo-EM is used to acquire 2D projection
images of thousands of individual, identical frozen-hydrated macromolecules
at random unknown orientations and positions. The collected images
are extremely noisy due to the limited electron dose used for imaging
to avoid excessive beam damage. In addition, the unknown orientational information of the imaged particles need to be estimated
for 3D reconstruction. An ab-initio estimation of the orientations of images using the random-conical tilt technique \cite{GWBP_radermacher} or common-lines based approaches \cite{VanHeel1987111, Amit_voting, Amit_eig_sdp} are often applied after multivariate statistical analysis \cite{Multivariatestat,vanHeel1981187} and classification techniques \cite{vanHeel1984165, GWBP_Penczek1992, Amit_classavg} that are used to sort and partition the large set of images by their viewing directions, producing ``class averages'' of enhanced signal-to-noise ratio (SNR). Using the ab-initio estimation of the orientations, a preliminary 3D map is reconstructed from the images by a 3D reconstruction algorithm. The initial model is then iteratively refined \cite{Refinement} in order to obtain a higher-resolution 3D reconstruction.
\begin{figure}

\begin{center}

\

\psfrag{text0}{Projection $P_i$}%

\psfrag{text1}{Projection $P_j$}%

\psfrag{text2}{$\hat{P}_i$}%

\psfrag{text3}{$\hat{P}_j$}%

\psfrag{text4}{3D Fourier space}%

\psfrag{text5}{3D Fourier space}%

\psfrag{(k0,l0)}[Bl][Bl][0.75]{$\vec{c}_{ij}=(c^1_{ij},c^2_{ij})$}%

\psfrag{(k1,l1)}[Bl][Bl][0.75]{$\vec{c}_{ji}=(c^1_{ji},c^2_{ji})$}%

\psfrag{Bk0l0}[Bl][Bl][0.75]{$R_{i}\left(\begin{array}{c}
\vec{c}_{ij}^T\\
0
\end{array}\right)$}%

\psfrag{Bk1l1}[Bl][Bl][0.75]{$R_{i}\left(\begin{array}{c}
\vec{c}_{ij}^T\\
0
\end{array}\right)=R_{j}\left(\begin{array}{c}
\vec{c}_{ji}^T\\
0
\end{array}\right)$}%

\includegraphics[width=0.5\paperwidth]{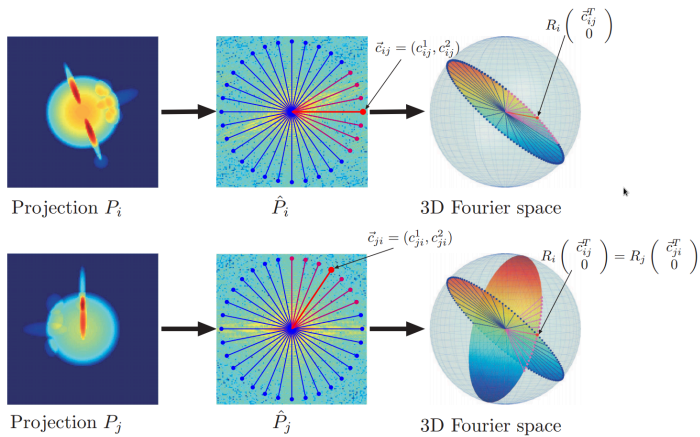}%

\caption{\label{fig:FPS-theorem} Fourier projection-slice theorem. In the middle, $\hat{P}_i$ is a polar Fourier transform of projection $P_i$ on the left. The red line $\vec {c}_{ij}$ represents the direction of a common-line between $\hat{P}_i$ and $\hat{P}_j$ on $\hat{P}_i$. On the right, the two transformed images $\hat{P}_i$ and $\hat{P}_j$ intersect with each other at the common-line after rotations $R_i$ and $R_j$, yielding the equation (\ref{eq:common-line-eq}). }

\end{center}

\end{figure}


The Fourier projection-slice theorem (see, e.g., \cite{natterer_fourier_proj_slice}) plays a fundamental role in the common-lines
based reconstruction methods. The theorem states that restricting the 3D Fourier transform of the volume to a planar central slice yields the Fourier transform
of a 2D projection of the volume in a direction perpendicular to the
slice (Figure \ref{fig:FPS-theorem}). Thus, any two projections imaged from non-parallel viewing directions
intersect at a line in Fourier space, which is called the common-line
between the two images. The common-lines between any three images
with linearly independent projection directions determine their relative orientation
up to handedness. This is the basis of the ``angular reconstitution''
technique of van Heel \cite{VanHeel1987111}, which was also developed independently by Vainshtein and Goncharov \cite{VainshteinGoncharov1986}. In this technique, the orientations of additional projections are determined in a sequential manner. Farrow and Ottensmeyer \cite{Farrow:92} used quaternions to obtain the relative orientation of a new projection in a least square sense. The main problem with such techniques is that
they are sensitive to false detection of common lines that leads to
the accumulation of errors. Penczek et.al. \cite{Penczek1996205} tried to obtain the rotations corresponding to all projections simultaneously by minimizing a global energy functional, which requires a brute force search in an exponentially large parametric space of all possible orientations for all projections. Mallick et. al. \cite{Mallick_structure} and Singer et al. \cite{Amit_voting}
applied Bayesian approaches to use common-lines information from different
groups of projections. Recently, Singer and Shkolnisky \cite{Amit_eig_sdp} developed
two algorithms based on eigenvectors and semidefinite programming for
estimating the orientations of all images. These two algorithms correspond to convex relaxations of the global self-consistency
error minimization, and can accurately estimate all orientations at
relatively low common-line detection rates.

When the signal-to-noise ratio (SNR) of the image is significantly low, the detected
common-lines consist of a modest number of noisy inliers, which are
explained well by the image orientations, along with a large number
of outliers, that have no structure. The standard common-lines based
methods, including those using least squares (LS) \cite{Farrow:92,Amit_eig_sdp}, are sensitive to
these outliers. In this paper we estimate the orientations using a different, more robust self consistency error, which is the sum of unsquared
residuals \cite{L1_Nyquist,L1_spath_watson}, rather than the sum of squared residuals of the LS formulation. Convex relaxations of least unsquared deviations (LUD) have been recently proposed for other applications, such as robust principal component analysis \cite{irls_t} and robust synchronization of orthogonal transformations \cite{LUD}. Under certain noise models for the distribution of the outliers (e.g., the haystack model of \cite{irls_t}), such convex relaxations enjoy proven guarantees for exact and stable recovery with high probability. Such theoretical and empirical improvements that LUD brings compared to LS serve as the main motivation to consider in this paper the application of LUD to the problem of orientation estimation from common-lines in single particle reconstruction.

The LUD minimization problem is solved here via semidefinite relaxation. When
the detection rate of common-lines is extremely low, the estimated viewing
directions of the projection images are observed to cluster together. This
artificial clustering can be explained by the fact that images that share the
same viewing direction also share more than one common line. In order to
mitigate this spurious clustering of estimated viewing directions, we add to the
minimization formulation a spectral norm term, either as a constraint or as
a regularization term. The resulting minimization problem is solved  by
the alternating direction method of multipliers (ADMM), which has been
proved to converge to the global minimizer in many cases \cite{HongLuo2012}. We also consider the application of the iteratively reweighted least squares (IRLS) procedure, which is not guaranteed to converge to the global minimizer, but performs well in our numerical experiments. We demonstrate that the ab-initio models resulted by our new methods are more accurate and require fewer refinement iterations compared to least squares based methods.

The paper is organized as follows: In Section \ref{sec:common} we review the
detection procedure of common lines between images. Section \ref{sec:LS}
presents the LS and LUD global self-consistency cost functions.  Section
\ref{sec:SDP} introduces the semidefinite relaxation and rounding procedure for
the LUD formulation. The additional spectral norm constraint is considered in
Section \ref{sec:spectral}. The ADMM method for obtaining the global minimizer is detailed in Section \ref{sec:adm}, and the IRLS procedure is described in Section \ref{sec:IRLS}. Numerical results for both simulated and real data are provided in Section \ref{sec:numerical_results}. Finally, Section \ref{sec:discussion} is a summary.

\section{Detection of common-lines between images}
\label{sec:common}

Typically, the first step for detecting common lines is to compute the 2D Fourier transform of each image on a polar grid using, e.g., the non-uniform fast Fourier transform (NUFFT) \cite{dutt:1368,nufft_fessler,greengard:443}. The transformed
images have resolution $n_{r}$ in the radial direction and resolution
$n_{\theta}$ in the angular direction, that is, the radial resolution
$n_{r}$ is the number of equi-spaced samples along each ray in the
radial direction, and the angular resolution $n_{\theta}$ is the
number of angularly equally-spaced Fourier rays computed for each image (Figure \ref{fig:FPS-theorem}).
For simplicity, we let $n_{\theta}$ be an even number. The transformed
images are denoted as $\left(\vec{l}_{0}^{k},\vec{l}_{1}^{k},\ldots,\vec{l}_{n_{\theta}-1}^{k}\right)$,
where $\vec{l}_{m}^{k}=\left(l_{m,1}^{k},l_{m,2}^{k},\ldots,l_{m,n_{r}}^{k}\right)$
is an $n_{r}$ dimensional vector, $m\in\left\{ 0,1,\ldots,n_{\theta}-1\right\} $
is the index of a ray, $k\in\left\{ 1,2,\ldots,K\right\} $ is the
index of an image and $K$ is the number of images. The DC term is
shared by all lines independently of the image, and is therefore excluded
for comparison. To determine the common line between two images $P_{i}$
and $P_{j}$, the similarity between all $n_{\theta}$ radial lines
$\vec{l}_{0}^{i},\vec{l}_{1}^{i},\ldots,\vec{l}_{n_{\theta}-1}^{i}$ from
the first image with all $n_{\theta}$ radial lines $\vec{l}_{0}^{j},\vec{l}_{1}^{j},\ldots,\vec{l}_{n_{\theta}-1}^{j}$ from
the second image are measured (overall $n_{\theta}^{2}$ comparisons),
and the pair of radial lines $\vec{l}_{m_{i,j}}^{i}$ and $\vec{l}_{m_{j,i}}^{j}$
with the highest similarity is declared as the common-line pair between
the two images. However, as a radial line is the complex conjugate
of its antipodal line, the similarity measure between $\vec{l}_{m_{1}}^{i}$
and $\vec{l}_{m_{2}}^{j}$ has the same value as that between their
antipodal lines $\vec{l}_{m_{1}+n_{\theta}/2}^{i}$ and $\vec{l}_{m_{2}+n_{\theta}/2}^{j}$
(where addition of indices is taken modulo $n_{\theta}$). Thus the
number of distinct similarity measures that need to be computed is
$n_{\theta}^{2}/2$ obtained by restricting the index $m_{1}$ to
take values between $0$ and $n_{\theta}/2$ and letting $m_{2}$
take any of the $n_{\theta}$ possibilities (see also \cite{VanHeel1987111}
and \cite{Penczek1994251}, p. 255). Equivalently, it is possible to
compare real valued 1D line projections of the 2D projection images,
instead of comparing radial Fourier lines that are complex valued.
According to the Fourier projection-slice theorem, each 1D projection
is obtained by the inverse Fourier transform of the corresponding Fourier radial line $\vec{l}_{m}^{k}$
and its antipodal line $\vec{l}_{m+n_{\theta}/2}^{k}$, and is denoted
as $\vec{s}_{m}^{k}$. The 1D projection lines of a cryo-EM image
can be displayed as a 2D image known as a ``sinogram''
(see  \cite{VanHeel1987111, Serysheva}).

Traditionally, the pair of radial lines (or sinogram lines) that has
the maximum normalized cross correlation is declared as the common
line, that is,
\begin{equation}
\left(m_{i,j},m_{j,i}\right)=\underset{{0 \leq m_{1} < n_{\theta}/2,\, 0 \leq m_{2} < n_{\theta}}}{\arg\max} \frac{\left\langle \vec{l}_{m_{1}}^{i},\vec{l}_{m_{2}}^{j}\right\rangle }{\left\Vert \vec{l}_{m_{1}}^{i}\right\Vert \left\Vert \vec{l}_{m_{2}}^{j}\right\Vert },\text{ for all }i\neq j,
\label{eq:commonline-estimate}
\end{equation}
where $m_{i,j}$ is a discrete estimate for where the $j$'th image
intersects with the $i$'th image. In practice, a weighted correlation,
which is equivalent to applying a combination of high-pass and low-pass
filters is used to determine proximity. As noted in \cite{VanHeel1987111},
the normalization is performed so that the correlation coefficient
becomes a more reliable measure of similarity between radial lines.
Note that even with clean images, this estimate will have a small
deviation from its ground truth (unknown) value due to discretization
errors. With noisy images, large deviations of the estimates from
their true values (say, errors of more than $10^{\circ}$) are frequent,
and their frequency increases with the level of noise. We refer to
common lines whose $m_{i,j}$ and $m_{j,i}$ values were estimated
accurately (up to a given discretization error tolerance) as ``correctly
detected" common lines, or ``inliers" and to the remaining common
lines as ``falsely detected", or ``outliers".

\section{Weighted LS and least unsquared deviation (LUD)}
\label{sec:LS}

We define the directions of detected common-lines between the transformed image
$i$ and transformed image $j$ as unit vectors (Figure \ref{fig:FPS-theorem})
\begin{eqnarray}
\vec{c}_{ij} & = & \left(c_{ij}^1,c_{ij}^2\right)=\left(\cos\left(2\pi m_{ij}/n_{\theta}\right),\sin\left(2\pi m_{ij}/n_{\theta}\right)\right),\label{eq:common-line-notation}\\
\vec{c}_{ji} & = & \left(c_{ji}^1,c_{ji}^2\right)=\left(\cos\left(2\pi m_{ji}/n_{\theta}\right),\sin\left(2\pi m_{ji}/n_{\theta}\right)\right),\label{eq:common-line-notation-1}
\end{eqnarray}
where $\vec{c}_{ij}$ and $\vec{c}_{ji}$ are on the transformed images
$i$ and $j$ respectively, and $m_{ij}$ and $m_{ji}$ are discrete
estimate for the common lines' positions using (\ref{eq:commonline-estimate}).
Let  the rotation matrices $R_{i}\in {\bf SO}(3)$, $i=1,\cdots,K$ represent
the orientations of the $K$ images. According to the Fourier projection-slice
theorem, the common lines on every two images should be the same after
the 2D transformed images are inserted in the 3D Fourier space using the corresponding
rotation matrices, that is,
\begin{equation}
R_{i}\left(\begin{array}{c}
\vec{c}_{ij}^T\\
0
\end{array}\right)=R_{j}\left(\begin{array}{c}
\vec{c}_{ji}^T\\
0
\end{array}\right)\;\text{for }1\leq i<j\leq K.\label{eq:common-line-eq}
\end{equation}
These can be viewed as $\left(\begin{array}{c}
K\\
2
\end{array}\right)$ linear equations for the $6K$ variables corresponding to the first
two columns of the rotation matrices (the third column of each rotation
matrix does not contribute in (\ref{eq:common-line-eq}) due to the
zero third entries in the common-line vectors in $\mathbb{R}^{3}$). The weighted LS approach for solving this system can be formulated as the minimization problem
\begin{equation}
\min_{R_{1},\ldots,R_{K}\in{\bf SO}\left(3\right)}\sum_{i\neq j}w_{ij}\left\Vert R_{i}\left(\vec{c}_{ij},0\right)^{T}-R_{j}\left(\vec{c}_{ji},0\right)^{T}\right\Vert ^{2},\label{eq:ls-exact}
\end{equation}
where the weights $w_{ij}$ indicate the confidence in the detections of common-lines between pairs of images. Since $\left(\vec{c}_{ij},0\right)^{T}$and $\left(\vec{c}_{ji},0\right)^{T}$
are 3D unit vectors, their rotations are also unit vectors; that is,
$\left\Vert R_{i}\left(\vec{c}_{ij},0\right)^{T}\right\Vert =\left\Vert R_{j}\left(\vec{c}_{ji},0\right)^{T}\right\Vert =1$.
It follows that the minimization problem (\ref{eq:ls-exact}) is equivalent
to the maximization problem of the sum of dot products
\begin{equation}
\max_{R_{1},\ldots,R_{K}\in{\bf SO} (3)}\sum_{i\neq j}w_{ij}\langle R_{i}\left(\vec{c}_{ij},0\right)^{T}, R_{j}\left(\vec{c}_{ji},0\right)^{T}\rangle. \label{eq:max-dotprod}
\end{equation}
When the
weight $w_{ij}=1$ for each pair $i\neq j$, (\ref{eq:max-dotprod}) is equivalent to the LS problem that was considered in \cite{Penczek1996205}, and more recently in \cite{Amit_eig_sdp} using convex relaxation of the non-convex constraint set. The solution to the LS problem
may not be optimal due to the typically large proportion
of outliers (Figure \ref{fig:fat-tails}).

\begin{figure}
\subfloat[SNR=1/32]{

\includegraphics[width=0.35\columnwidth]{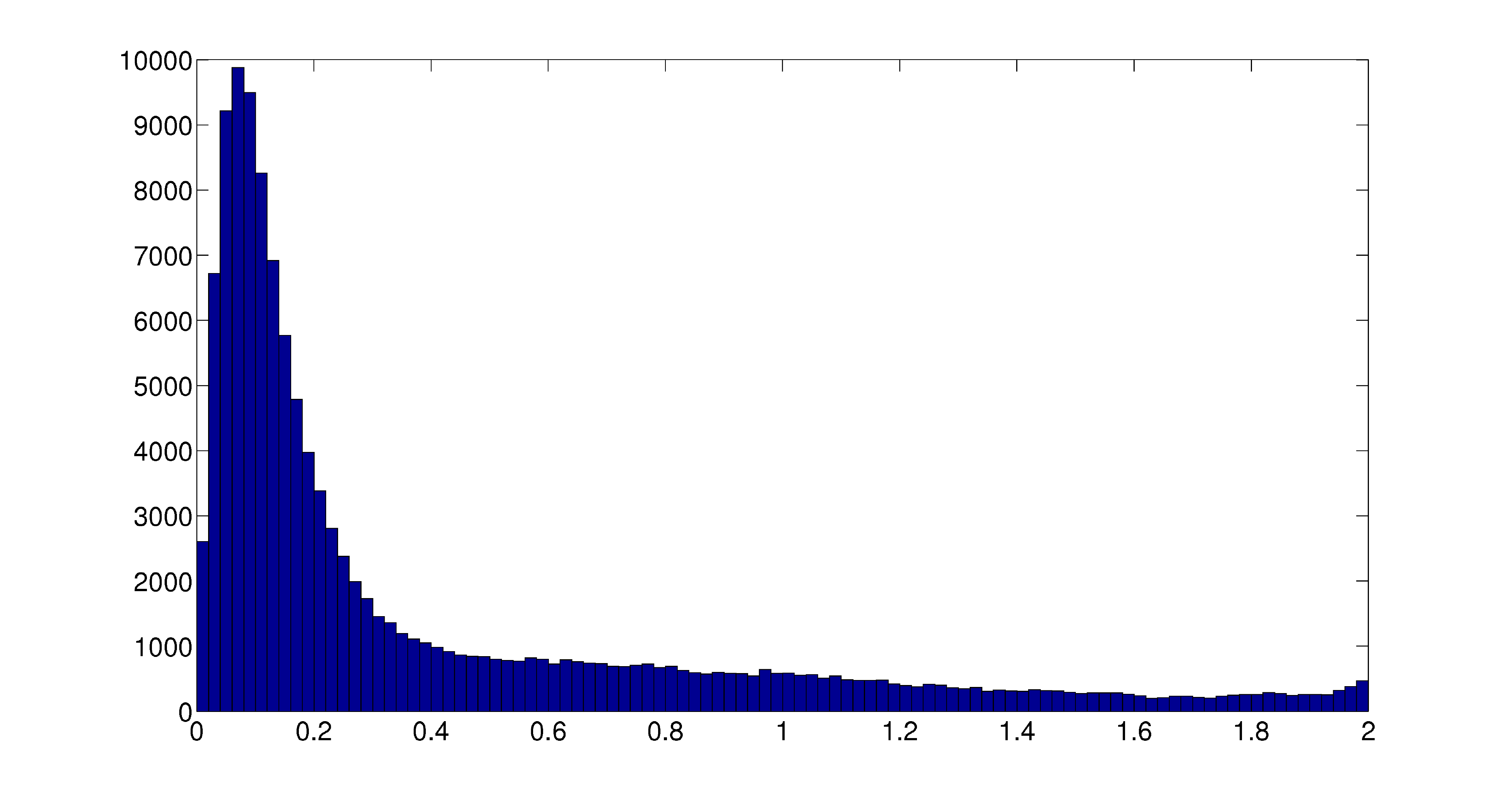}}\hfill{}\subfloat[SNR=1/64]{

\includegraphics[width=0.35\columnwidth]{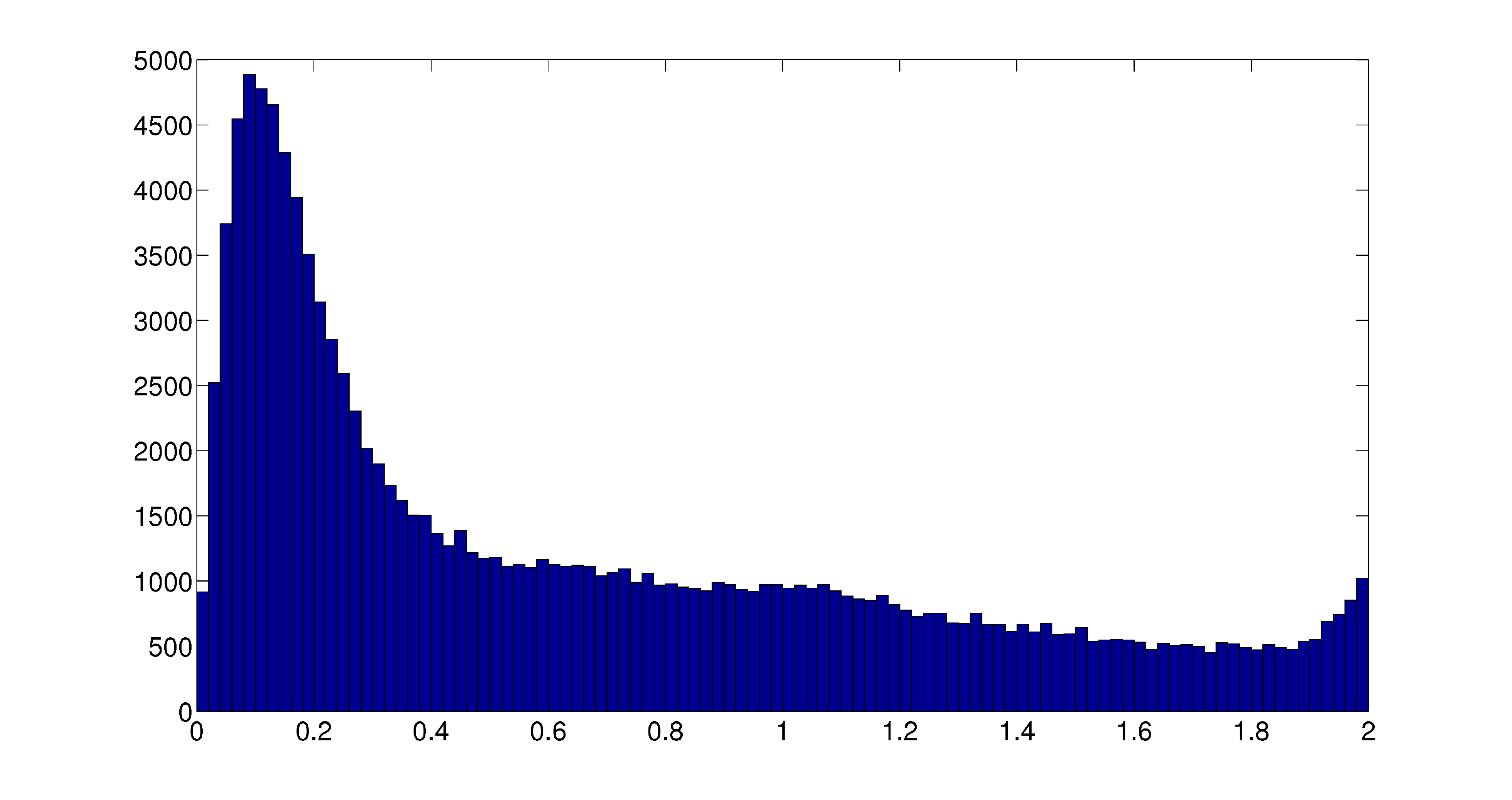}}\hfill{}\subfloat[\label{fig:L2vsL1}$x^2$ vs $\left|x\right| $]{

\includegraphics[width=0.2\columnwidth]{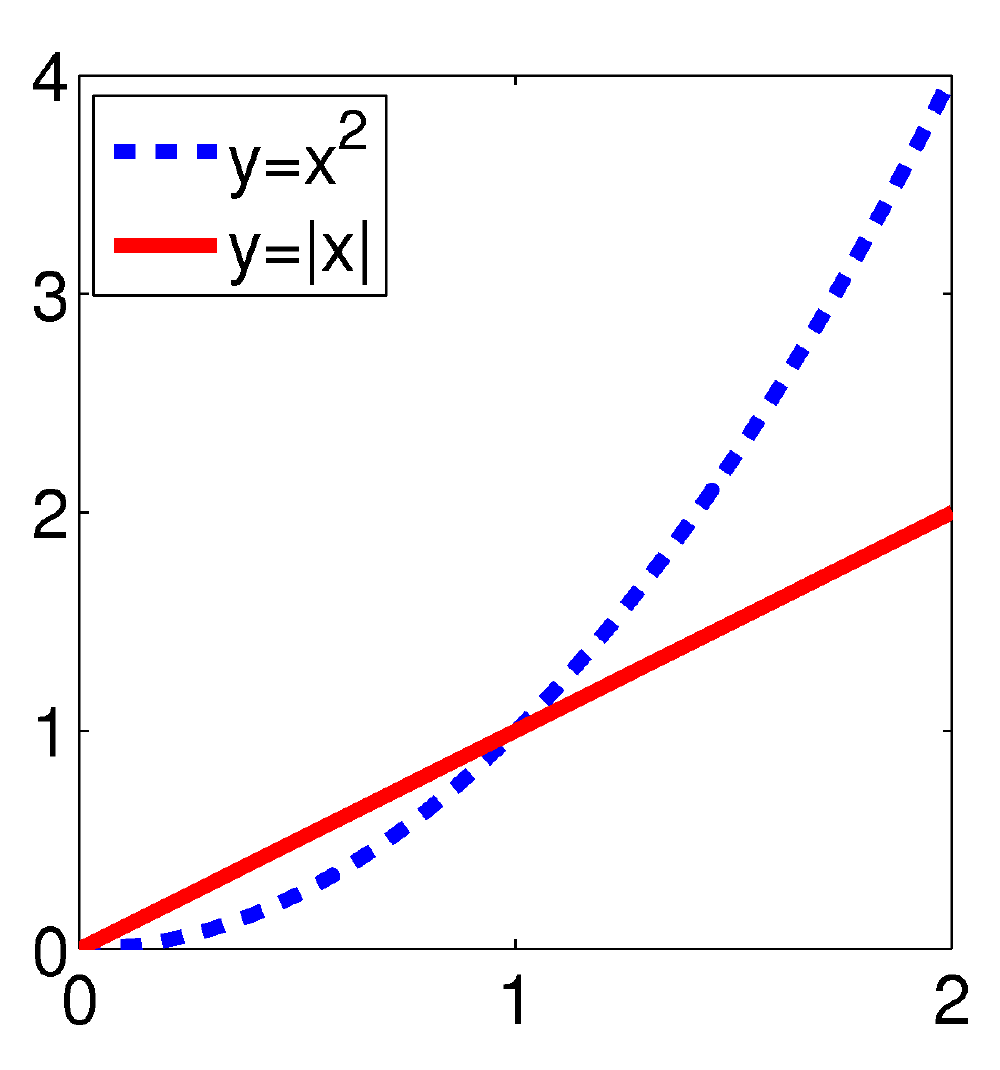}}\hfill{}

\caption{\label{fig:fat-tails}Left and Middle:
The histogram plots of errors in the detected common-lines $\vec{c}_{ij}$ for all $i$ and $j$, i.e., $\left\Vert R_{i}\left(\vec{c}_{ij},0\right)^{T}-R_{j}\left(\vec{c}_{ji},0\right)^{T}\right\Vert$ where $R_i$ is a true rotation matrix for all $i$. The fat tail in (b) indicates the detected common-lines contain a large amount of outliers. Right: elucidating the difference between the squared distance and the absolute deviation.}
\end{figure}

To guard the estimation of the orientations from outliers, we replace
the sum of weighted squared residuals in (\ref{eq:ls-exact}) with the more robust sum of unsquared residuals and obtain
\begin{equation}
\min_{R_{1},\ldots,R_{K}\in{\bf SO}\left(3\right)}\sum_{i\neq j}\left\Vert R_{i}\left(\vec{c}_{ij},0\right)^{T}-R_{j}\left(\vec{c}_{ji},0\right)^{T}\right\Vert \label{eq:lad-exact},
\end{equation}
or equivalently,
\begin{equation}
\min_{R_{1},\ldots,R_{K}\in{\bf SO}\left(3\right)}\sum_{i\neq j}\left\Vert \left(\vec{c}_{ij},0\right)^{T}-R_{i}^T R_{j}\left(\vec{c}_{ji},0\right)^{T}\right\Vert \label{eq:lad-exact_form2}.
\end{equation}
We refer to the minimization problem (\ref{eq:lad-exact}) as the least unsquared deviation (LUD) problem.
The self consistency error given in (\ref{eq:lad-exact}) reduces the contribution from large
residuals that may result from outliers (Figure {\ref{fig:L2vsL1}). We remark that it is also possible to consider the weighted version of (\ref{eq:lad-exact}), namely $$\min_{R_{1},\ldots,R_{K}\in{\bf SO}\left(3\right)}\sum_{i\neq j}w_{ij}\left\Vert R_{i}\left(\vec{c}_{ij},0\right)^{T}-R_{j}\left(\vec{c}_{ji},0\right)^{T}\right\Vert.$$ For simplicity, we focus here on the unweighted version.

\section{Semidefinite Programming Relaxation (SDR) and the Rounding Procedure}
\label{sec:SDP}

Both the weighted LS problem (\ref{eq:ls-exact}) and the LUD problem (\ref{eq:lad-exact})
are non-convex and therefore extremely difficult to solve if one requires
the matrices $R_{i}$ to be rotations, that is, when adding the constraints
\begin{equation}
R_{i}R_{i}^{T}=I_{3},\;\det\left(R_{i}\right)=1,\;\;\text{ for }i=1,\ldots,K,\label{eq:rotation-constraint}
\end{equation}
where $I_{3}$ is the $3\times3$ identity matrix. A relaxation method
that neglects the constraints (\ref{eq:rotation-constraint}) will
simply collapse to the trivial solution $R_{1}=\ldots=R_{K}=0$ which
obviously does not satisfy the constraint (\ref{eq:rotation-constraint}).

The relaxation in \cite{Amit_eig_sdp} that uses semidefinite programming (SDP) can be modified in a straightforward manner in order to deal with non-unity weights $w_{ij}$ in (\ref{eq:max-dotprod}). We present this modification here for three reasons. First, the weighted version is required by the IRLS procedure (see Section \ref{sec:IRLS}). Second, the rounding procedure after SDP employed here is slightly different than the one presented in \cite{Amit_eig_sdp} and is closer in spirit to the rounding procedure of Goemans and Williamson for the MAX-CUT problem \cite{GW_SDP}. Finally, in the Appendix we prove exact recovery of the rotations by the semidefinite relaxation procedure when the detected common-lines are all correct.

\subsection{Constructing the Gram matrix $G$ from the rotations $R_i$}
We denote the columns of the rotation matrix $R_{i}$ by $R_{i}^{1}$, $R_{i}^{2}$, and $R_{i}^{3}$, and write the rotation matrices
as
\[
R_{i}=\left(\begin{array}{ccc}
| & | & |\\
R_{i}^{1} & R_{i}^{2} & R_{i}^{3}\\
| & | & |
\end{array}\right),\;\; i=1,\ldots,K.
\]
We define a $3\times2K$ matrix $R$ by concatenating the first two columns of all rotation matrices:
\begin{equation}
R=\left(\begin{array}{cccc}
| & | &  & |\\
R_{1}^{1} & R_{1}^{2} & \cdots & R_{k}^{1}\\
| & | &  & |
\end{array}\begin{array}{cccc}
| & &  | & |\\
R_{k}^{2} & \cdots & R_{K}^{1} & R_{K}^{2}\\
| &  & | & |
\end{array}\right).\label{eq:R}
\end{equation}
The Gram matrix $G$ for the matrix $R$ is a $2K\times2K$ matrix
of inner products between the 3D column vectors of $R$, that is,
\begin{equation}
G=R^{T}R.\label{eq:G}
\end{equation}
Clearly, $G$ is a rank-$3$ semidefinite positive matrix ($G\succcurlyeq0$),
which can be conveniently written as a block matrix
\[
G=\left(G_{ij}\right)_{i,j=1,\cdots,K},
\]
where $G_{ij}$ is the $2\times 2$ upper left block of the rotation matrix $R_i^T R_j$, that is,
\[
G_{ij} = \left(\begin{array}{c}
(R_i^1)^T\\
(R_i^2)^T
\end{array}\right)\left(\begin{array}{cc}R_i^1 & R_i^2 \end{array}\right).
\]
In addition, the orthogonality of the rotation matrices ($R_{i}^{T}R_{i}=I$) implies
that
\begin{equation}
G_{ii}=I_2,\text{ }i=1,2,\cdots,K,\label{eq:orth-constraint}
\end{equation}
where $I_2$ is the $2\times 2$ identity matrix.

\subsection{SDR for weighted LS}
\label{sec:sdr_ls}
We
first define two $2K\times 2K$ matrices $S=(S_{ij})_{i,j=1,\cdots,K}$ and $W=(W_{ij})_{i,j=1,\cdots,K}$, where the $2\times2$ sub-blocks $S_{ij}$ and $W_{ij}$ are given by
\[
S_{ij}=\vec{c}_{ji}^T \vec{c}_{ij},\]
and
\[W_{ij} = w_{ij}\left(\begin{array}{cc}
1&1\\
1&1
\end{array}\right).
\]
Both matrices $S$  and $W$ are symmetric and they store all available common-line information and weight information, respectively. It follows that the objective
function (\ref{eq:max-dotprod}) is the trace of the matrix $\left(W\circ S\right)G$:
\begin{equation}
\sum_{i\neq j}w_{ij}\langle R_{i}\left(\vec{c}_{ij},0\right)^{T}, R_{j}\left(\vec{c}_{ji},0\right)^{T}\rangle=\text{trace}\left(\left(W\circ S\right)G\right),\label{eq:trace}
\end{equation}
where the symbol $\circ$ denotes the Hadamard product between two
matrices. A natural relaxation of the optimization problem (\ref{eq:max-dotprod})
is thus given by the SDP problem
\begin{eqnarray}
 & \max_{G\in\mathbb{R}^{2K\times2K}} & \text{trace}\left(\left(W\circ S\right)G\right)\label{eq:sdp}\\
\text{s.t. } & &G_{ii}=I_2,\text{ }i=1,2,\cdots,K,\label{eq:sdp_orth_constraint}\\
 &  & G\succcurlyeq0\label{eq:sdp_constraint}
\end{eqnarray}
The non-convex
rank-$3$ constraint on the Gram matrix $G$ is missing from this semidefinite relaxation (SDR) \cite{SDR}. The problem (\ref{eq:sdp})-(\ref{eq:sdp_constraint}) is an SDP that can be solved by standard SDP solvers. In particular, it can be well solved by the solver SDPLR \cite{Burer01anonlinear} which takes advantage of the low-rank property of $G$. SDPLR is a first-order algorithm via low-rank factorization and hence can provide approximate solutions for large scale problems. Moreover, the iterations of SDPLR are extremely fast.

\subsection{SDR for LUD}
Similar to defining the Gram matrix $G$ in (\ref{eq:G}), we define a $3K\times 3K$ matrix $\tilde{G}$ as $\tilde{G}=(\tilde{G}_{ij})_{i,j=1,\cdots,K}$, where each $\tilde{G}_{ij}$ is a $3\times 3$ block defined as $\tilde{G}_{ij} = R_i^T R_j$. Then, a natural SDR for (\ref{eq:lad-exact_form2}) is given by
\begin{equation}
\min_{\tilde{G} \succcurlyeq 0}\sum_{i\neq j}\left\Vert \left(\vec{c}_{ij},0\right)^{T}-\tilde{G}_{ij}\left(\vec{c}_{ji},0\right)^{T}\right\Vert,\text{ s.t. } \tilde{G}_{ii}=I_3. \label{eq:lad-exact_sdr}
\end{equation}
The constraints missing in this SDP formulation are the non-convex
rank-$3$ constraint and the determinant constraints $\text{det}(\tilde{G}_{ij})=1$ on the Gram matrix $\tilde{G}$.
However, the solution $\tilde{G}$ to (\ref{eq:lad-exact_sdr}) is not unique.
Note that if a set of rotation matrices $\{R_i\}$ is the solution to (\ref{eq:lad-exact_form2}), then the set of conjugated rotation matrices $\{JR_i J\}$ is also the solution to (\ref{eq:lad-exact_form2}), where the matrix $J$ is defined as
\[
J=\left(\begin{array}{ccc}
1&0&0\\
0&1&0\\
0&0&-1
\end{array}\right).
\]
Thus, another solution to (\ref{eq:lad-exact_sdr}) is the Gram matrix $\tilde{G}^J=(\tilde{G}_{ij}^J)_{i,j=1,\cdots,K}$ with the $3\times 3$ sub-blocks given by $\tilde{G}_{ij} ^J=J R_i^T J J R_j J = J R_i^T R_j J$. It can be verified that $\frac{1}{2}(\tilde{G}+\tilde{G}^J)$ is also a solution to (\ref{eq:lad-exact_sdr}). Using the fact that
\[
\frac{1}{2}(\tilde{G}_{ij}+\tilde{G}_{ij}^J)=\left(\begin{array}{cc}
G_{ij} & \begin{array}{c}
0\\
0
\end{array}\\
\begin{array}{cc}
0 & 0\end{array} & 0
\end{array}\right),
\]
the problem (\ref{eq:lad-exact_sdr}) is reduced to
\begin{equation}
\min_{G\succcurlyeq 0}\sum_{i\neq j}\left\Vert \vec{c}_{ij}^{T}-G_{ij}\vec{c}_{ji}^{T}\right\Vert,\text{ s.t. } G_{ii}=I_2. \label{eq:lad-exact_sdr_2}
\end{equation}
This is a SDR for the LUD problem (\ref{eq:lad-exact}). The problem (\ref{eq:lad-exact_sdr_2}) can be solved using ADMM (see details in section \ref{sec:adm_lud}).
\subsection{The Randomized Rounding Procedure}
\label{sec:rounding}
The matrix $R$ is recovered
from a random projection of the solution $G$ of the SDP (\ref{eq:sdp}).
We randomly draw a $2K\times3$ matrix $P$ from the Stiefel manifold
$V_{3}(\mathbb{R}^{2K})$. The random matrix $P$ is computed using
the orthogonal matrix $Q$ and the upper triangular matrix $R$ from
QR factorization of a random matrix with standard i.i.d Gaussian entries,
that is, $P=Q \text{ sign}\left(\text{diag}\left(R\right)\right)$,
where sign stands for the entry-wise sign function and diag$\left(R\right)$ is a diagonal
matrix whose diagonal entries are the same as those of the matrix $R$.
The matrix $P$ is shown to be drawn uniformly from the Stiefel manifold
in  \cite{Mezzadri:2007}. We project the solution $G$ onto the subspace spanned
by the three columns of the matrix $GP$ \footnote{The 3 dimensional subspace can also be spanned by the eigenvectors associated with the top three eigenvalues of $G$, while the fourth largest eigenvalue is expected to be significantly smaller; see also \cite{Amit_eig_sdp}.}.

The $2K\times 3$ matrix $GP$ is a proxy to the matrix $R^T$ (up to a global $3\times 3$ orthogonal transformation). In other words, we can regard the $3\times 2K$ matrix $(GP)^T$ as composed from $K$ matrices of size $3\times 2$, denoted $A_i$ ($i=1,\ldots,K$), namely,
$$(GP)^T = \left(\begin{array}{cccc} A_1 & A_2 & \cdots & A_K
\end{array} \right)$$
The two columns of each $A_i$ correspond to $R_i^1$ and $R_i^2$ (compare to (\ref{eq:R})). We therefore estimate the matrix
$R_{i}^{[1,2]} =\left(\begin{array}{cc}                                                                                                                                                                                   R_i^1 & R_i^2                                                                                                                                                                                 \end{array}\right)$
as the closest matrix to $A_{i}$
on the Stiefel manifold $V_{2}(\mathbb{R}^{3})$ in the Frobenius
matrix norm. The closest matrix is given by (see, e.g., \cite{Arun:1987:LFT:28809.28821}) $R_{i}^{[1,2]}=U_{i}V_{i}^{T}$,
where $A_{i}=U_{i}\Sigma_{i}V_{i}^{T}$ is the singular value
decomposition of $A_{i}$. We note that except for the orthogonality
constraint (\ref{eq:sdp_orth_constraint}), the semidefinite program
(\ref{eq:sdp})\textendash{}(\ref{eq:sdp_constraint}) is identical
to the Goemans\textendash{}Williamson SDP for finding the maximum
cut in a weighted graph \cite{GW_SDP}, where the SDR and the randomized rounding procedure \cite{SDP,SDR} for maximum cut problem is proved to have a $0.87$ performance guarantee. From the complexity point of view,
SDP can be solved in polynomial time to any given precision. The idea
of using SDP for determining image orientations in cryo-EM was originally
proposed in \cite{Amit_eig_sdp}.

\section{The Spectral Norm Constraint}
\label{sec:spectral}

In our numerical experiments (see Section \ref{sec:numerical_results}), we observed that in the presence of many ``outliers" (i.e., a large proportion of misidentified common-lines), the estimated viewing directions\footnote{The viewing direction is the third column of the underlying rotation matrix.} that are obtained by either solving (\ref{eq:sdp})-(\ref{eq:sdp_constraint}) or (\ref{eq:lad-exact_sdr_2}) are highly clustered (Figure \ref{fig:alpha}). This empirical behavior of the solutions can be explained by the fact that images whose viewing directions are parallel share many common lines. In other words, when the viewing directions of $R_i$ and $R_j$ are nearby, the fidelity term $\left\Vert R_{i}\left(\vec{c}_{ij},0\right)^{T}-R_{j}\left(\vec{c}_{ji},0\right)^{T}\right\Vert$ (that appears in all cost functions) can become small (i.e., close to 0), even when the common line pair $(\vec{c}_{ij},\vec{c}_{ji})$ is misidentified.

\begin{figure}
\center
    \includegraphics[width=0.65\paperwidth]{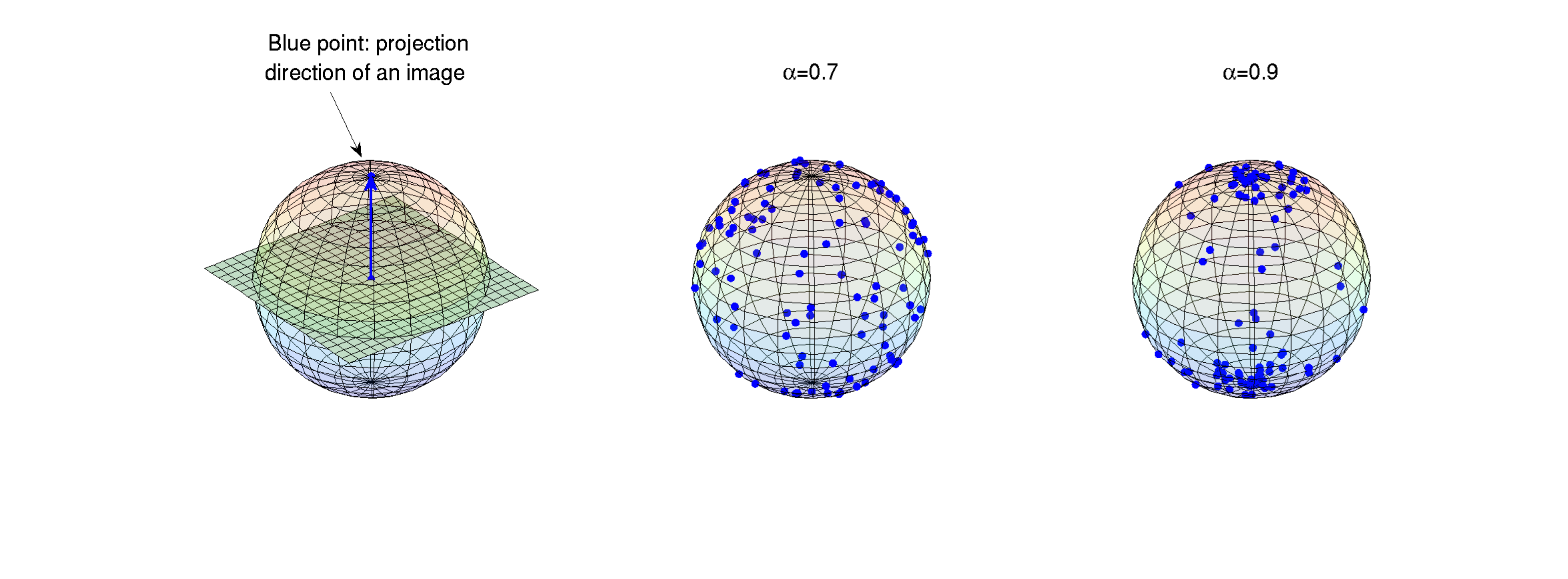}
 \caption{\label{fig:alpha}The dependency of the spectral norm of $G$ (denoted as $\alpha K$ here) on the distribution of orientations of the images. Here $K=100$. The larger $\alpha$ is, the more clustered the orientations are.}
\end{figure}

In order to prevent the viewing directions from clustering, we add the following constraint on the spectral norm of the Gram matrix $G$ to the optimization problem (\ref{eq:sdp})-(\ref{eq:sdp_constraint}) or (\ref{eq:lad-exact_sdr_2}):
\begin{equation}
G\preccurlyeq \alpha K I_{2K},
\end{equation}
where $I_{2K}$ is the $2K\times 2K$ identity matrix, or equivalently
\begin{equation}
\left\Vert G\right\Vert_2\leq \alpha  K,\label{eq:spectral_constraint}
\end{equation}
where $\left\Vert G\right\Vert_2$ is the spectral norm of the matrix $G$, and
the parameter $\alpha \in [\frac{2}{3},1)$ controls the spread of the viewing
directions. If the true image orientations are uniformly sampled from the
rotation group ${\bf SO(3)}$, then by the law of large numbers and the symmetry
of the distribution of orientations, the spectral norm of the true Gram matrix
$G_{true}$ is approximately $\frac{2}{3}K$ (To see this, notice that
$\operatorname{Tr}(G) = \operatorname{Tr}(R^TR) = \operatorname{Tr}(RR^T) =
\operatorname{Tr}(KI_2) = 2K$. Thus, the sum of eigenvalues of $G$ is $2K$.
Recall that $G$ is of rank 3, so if the rotations are uniformly distributed then
each of its three non-trivial eigenvalues equals $\frac{2K}{3}$). On the other
hand, if the true viewing directions are highly clustered, then the spectral
norm of the true Gram matrix $G_{true}$ is close to $K$. For a known
distribution of orientations, we can compute the spectral norm of the true Gram
matrix $G_{true}$ accordingly, which can be verified to be a number between
$\frac{2}{3}$ and $1$. In practice, however, the distribution of the viewing
directions is usually unknown a-priori, and often it cannot be assumed to be
uniform. To prevent a solution with clustered viewing directions, we fix the
parameter $\alpha$ to some number satisfying $\frac{2}{3}\leq\alpha<1$, and
perhaps even try a few possible values for $\alpha$ and choose the best value by examining the resulting reconstructions.

\section{The Alternating Direction Method of Multipliers (ADMM) for SDRs with Spectral Norm Constraint}
\label{sec:adm}

The application of ADMM to SDP problems was considered in \cite{ADM}. Here we generalize the application of ADMM to the optimization problems considered in previous sections. ADMM is a multiple-splitting algorithm that minimizes the augmented Lagrangian function in an alternating fashion such that in each step it minimizes over one block of the variables with all other blocks fixed, and then update the Lagrange multipliers. We apply ADMM to the dual problems since the linear constraints (\ref{eq:adm_linear}) satisfy $\mathcal{AA}^{*}=I$ which simplifies the computation of subproblems. The strong duality theorem, which is known as Slater's theorem, guarantees that in the presence of a strictly feasible solution, a primal problem can be solved by solving its dual problem. To obtain a strictly feasible solution to the primal problems with the positive semidefinite constraint, the linear constraint  (\ref{eq:adm_linear}) and  the spectral norm constraint (\ref{eq:adm_spectral}), we can construct a Gram matrix G in (\ref{eq:G}) using rotations sampled from a uniform distribution over the rotation group. Therefore, strong duality holds for the primal problems, and the primal problems can be solved by applying ADMM to their corresponding dual problems.

\subsection{The relaxed weighted LS problem}
\label{sec:adm_wls}
The weighted LS problem after SDR (\ref{eq:sdp})-(\ref{eq:sdp_constraint}) can be efficiently solved using SDPLR \cite{Burer01anonlinear}. However, SDPLR is not suitable for the problem after the spectral norm constraint on $G$ (\ref{eq:spectral_constraint}) is added to (\ref{eq:sdp})-(\ref{eq:sdp_constraint}). This is because the constraint  (\ref{eq:spectral_constraint}) can be written as $\alpha K I - G\succcurlyeq0$, but $\alpha K I -G$ does not have a low rank structure. Moreover, SDP solvers using polynomial-time primal-dual interior point methods are designed for small to medium sized problems. Therefore, they are not suitable for our problem. Instead, we devise here a version of ADMM which takes advantage of the low-rank property of $G$. After the spectral norm constraint (\ref{eq:spectral_constraint}) is added, the problem (\ref{eq:sdp})-(\ref{eq:sdp_constraint}) becomes
\begin{eqnarray}
\min_{G\succcurlyeq0} & -\left\langle C,G\right\rangle \label{eq:adm_obj}\\
\text{s.t.} & \mathcal{A}\left(G\right)={\bf b}\label{eq:adm_linear}\\
 & \left\Vert G\right\Vert _{2}\leq\alpha K\label{eq:adm_spectral}
\end{eqnarray}
where
\begin{equation}
\mathcal{A}\left(G\right)=\left(\begin{array}{c}
G_{ii}^{11}\\
G_{ii}^{22}\\
\frac{\sqrt{2}}{2}G_{ii}^{12}+\frac{\sqrt{2}}{2}G_{ii}^{21}
\end{array}\right)_{i=1,2,\ldots,K},\,{\bf b}=\left(\begin{array}{c}
b_{i}^{1}\\
b_{i}^{2}\\
b_{i}^{3}
\end{array}\right)_{i=1,2,\ldots,K},\label{eq:Ab}
\end{equation}
\[
b_{i}^{1}=b_{i}^{2}=1,\: b_{i}^{3}=0\text{ for all }i,
\]
$G_{ij}^{pq}$ denotes the $(p,q)$ th element in the $2\times2$ sub-block $G_{ij}$, $C=W\circ S$ is a symmetric matrix and $\left\langle C,G\right\rangle =\text{trace}\left(CG\right)$.
Following the equality $\left\langle \mathcal{A}\left(G\right),{\bf y}\right\rangle =\left\langle G,\mathcal{A}^{*}\left({\bf y}\right)\right\rangle $
for arbitrary ${\bf y}=\left(\begin{array}{c}
y_{i}^{1}\\
y_{i}^{2}\\
y_{i}^{3}
\end{array}\right)_{i=1,2,\ldots,K}$, the adjoint of the operator $\mathcal{A}$ is defined as
\[
\mathcal{A}^{*}\left({\bf y}\right)=Y=\left(\begin{array}{cc}
Y^{11} & Y^{12}\\
Y^{21} & Y^{22}
\end{array}\right),
\]
where for $i=1,2,\ldots,K$
\[
Y_{ii}^{11}=y_{i}^{1},\: Y_{ii}^{22}=y_{i}^{2},\text{ and }Y_{ii}^{12}=Y_{ii}^{21}=y_{i}^{3}/\sqrt{2}.
\]
It can be verified that $\mathcal{AA}^{*}=I$. The dual problem of
problem (\ref{eq:adm_obj})-(\ref{eq:adm_spectral}) is
\begin{equation}
\max_{{\bf y},X\succcurlyeq0}\min_{\left\Vert G\right\Vert _{2}\leq\alpha K}-\left\langle C,G\right\rangle -\left\langle {\bf y},\mathcal{A}\left(G\right)-{\bf b}\right\rangle -\left\langle G,X\right\rangle .\label{eq:dual1}
\end{equation}
By rearranging terms in (\ref{eq:dual1}), we obtain
\begin{equation}
\max_{{\bf y},X\succcurlyeq0}\min_{\left\Vert G\right\Vert _{2}\leq\alpha K}-\left\langle C+X+\mathcal{A}^{*}\left({\bf y}\right),G\right\rangle +{\bf y}^{T}{\bf b}.\label{eq:dual2}
\end{equation}
Using the fact that the dual norm of the spectral norm is the nuclear
norm (Proposition 2.1 in \cite{Recht2010}), we can obtain  from (\ref{eq:dual2})
the dual problem
\begin{equation}
\max_{{\bf y},X\succcurlyeq0}\,{\bf y}^{T}{\bf b}-\alpha K\left\Vert C+X+\mathcal{A}^{*}\left({\bf y}\right)\right\Vert _{*},\label{eq:dual3}
\end{equation}
where $\left\Vert \cdot\right\Vert _{*}$ denotes the nuclear norm.
Introducing a variable $Z=C+X+\mathcal{A}^{*}\left({\bf y}\right),$
we obtain from (\ref{eq:dual3}) that
\begin{eqnarray}
\min_{{\bf y},X\succcurlyeq0} & -{\bf y}^{T}{\bf b}+\alpha K\left\Vert Z\right\Vert _{*}\label{eq:dual4}\\
\text{s.t.} \;& Z=C+X+\mathcal{A}^{*}\left({\bf y}\right).\label{eq:dual5}
\end{eqnarray}
Since $Z$ is a symmetric matrix, $\left\Vert Z\right\Vert _{*}$
is the summation of the absolute values of the eigenvalues of $Z$.
The augmented Lagrangian function of (\ref{eq:dual4})-(\ref{eq:dual5})
is defined as
\begin{eqnarray}
\mathcal{L}\left({\bf y},Z,X,G\right)=&-&{\bf y}^{T}{\bf b}+\alpha K\left\Vert Z\right\Vert _{*}+\left\langle G,C+X+\mathcal{A}^{*}\left({\bf y}\right)-Z\right\rangle \nonumber\\
&+&\frac{\mu}{2}\left\Vert C+X+\mathcal{A}^{*}\left({\bf y}\right)-Z\right\Vert _{F}^{2},\label{eq:Lagrangian}
\end{eqnarray}
where $\mu>0$ is a penalty parameter. Using the augmented Lagrangian
function (\ref{eq:Lagrangian}), we devise an ADMM that minimizes (\ref{eq:Lagrangian}) with respect to
${\bf y}$, $Z$, $X$, and $G$ in an alternating fashion, that is,
given some initial guess, in each iteration the following three subproblems
are solved sequentially:
\begin{eqnarray}
{\bf y}^{k+1} & = & \arg\min_{{\bf y}}\mathcal{L}\left({\bf y},Z^{k},X^{k},G^{k}\right),\label{eq:sub_y}\\
Z^{k+1} & = & \arg\min_{Z}\mathcal{L}\left({\bf y}^{k+1},Z,X^{k},G^{k}\right),\label{eq:sub_Z}\\
X^{k+1} & = & \arg\min_{X\succcurlyeq0}\mathcal{L}\left({\bf y}^{k+1},Z^{k+1},X,G^{k}\right),\label{eq:sub_X}
\end{eqnarray}
and the Lagrange multiplier $G$ is updated by
\begin{equation}
G^{k+1}=G^{k}+\gamma\mu\left(C+X^{k+1}+\mathcal{A}^{*}\left({\bf y}^{k+1}\right)-Z^{k+1}\right),\label{eq:G_update}
\end{equation}
where $\gamma\in\left(0,\frac{1+\sqrt{5}}{2}\right)$ is an appropriately
chosen step length.

To solve the subproblem (\ref{eq:sub_y}), we use the first order
optimality condition
\begin{equation*}
\nabla_{{\bf y}}\mathcal{L}\left({\bf y},Z^{k},X^{k},G^{k}\right)=0
\end{equation*}
and the fact that $\mathcal{AA}^{*}=I$, and we obtain
\[
{\bf y}^{k+1}=-\mathcal{A}\left(C+X^{k}-Z^{k}\right)-\frac{1}{\mu}\left(\mathcal{A}\left(G\right)-{\bf b}\right).
\]

By rearranging the terms of $\mathcal{L}\left({\bf y}^{k+1},Z,X^{k},G^{k}\right)$,
it can be verified that the subproblem (\ref{eq:sub_Z}) is equivalent
to
\[
\min_{Z}\frac{\alpha K}{\mu}\left\Vert Z\right\Vert _{*}+\frac{1}{2}\left\Vert Z-B^{k}\right\Vert _{F}^{2},
\]
where $B^{k}=C+X^{k}+\mathcal{A}^{*}\left({\bf y}^{k+1}\right)+\frac{1}{\mu}G^{k}$.
Let $B^{k}=U\Lambda U^{T}$ be the spectral decomposition of the matrix
$B^{k}$, where $\Lambda=\text{diag}\left(\boldsymbol{\lambda}\right)=\text{diag}\left(\lambda_{1},\ldots,\lambda_{2K}\right).$
Then $Z^{k+1}=U\text{diag}\left(\hat{{\bf z}}\right)U^{T},$ where
$\hat{{\bf z}}$ is the optimal solution of the problem
\begin{equation}
\min_{{\bf z}}\frac{\alpha K}{\mu}\left\Vert {\bf z}\right\Vert _{1}+\frac{1}{2}\left\Vert {\bf z}-\boldsymbol{\lambda}\right\Vert _{2}^{2},\label{eq:soft_th}
\end{equation}
It can be shown that the unique solution of (\ref{eq:soft_th}) admits
a closed form called the soft-thresholding operator, following a terminology
introduced by Donoho and Johnstone \cite{soft_th}; it can be written
as
\[
\hat{z}_{i}=\begin{cases}
0, & \text{if }\left|\lambda_{i}\right|\leq\alpha K/\mu\\
(1-\frac{\alpha}{\mu}/|\lambda_{i}|)\lambda_{i}, & \text{otherwise.}
\end{cases}
\]

The problem (\ref{eq:sub_X}) can be shown to be equivalent to
\[
\min_{X}\left\Vert X-H^{k}\right\Vert _{F}^{2},\text{ s.t. }X\succcurlyeq0,
\]
where $H^{k}=Z^{k+1}-C-\mathcal{A}^{*}\left({\bf y}^{k+1}\right)-\frac{1}{\mu}G^{k}.$
The solution $X^{k+1}=V_{+}\Sigma_{+}V_{+}^{T}$ is the Euclidean
projection of $H^{k}$ onto the semidefinite cone (section 8.1.1 in \cite{Boyd}), where
\[
V\Sigma V^{T}=\left(\begin{array}{cc}
V_{+} & V_{-}\end{array}\right)\left(\begin{array}{cc}
\Sigma_{+} & 0\\
0 & \Sigma_{-}
\end{array}\right)\left(\begin{array}{c}
V_{+}^{T}\\
V_{-}^{T}
\end{array}\right)
\]
is the spectral decomposition of the matrix $H^{k},$ and $\Sigma_{+}$
and $\Sigma_{-}$ are the positive and negative eigenvalues of $H^{k}.$

It follows from the update rule (\ref{eq:G_update}) that
\begin{eqnarray*}
G^{k+1} & = & (1-\gamma)G^{k}+\gamma\mu\left(C+X^{k+1}+\mathcal{A}^{*}\left({\bf y}^{k+1}\right)-Z^{k+1}+\frac{1}{\mu}G^{k}\right)\\
 & = & (1-\gamma)G^{k}+\gamma\mu\left(X^{k+1}-H^{k}\right).
\end{eqnarray*}

\subsection{The relaxed LUD problem}
\label{sec:adm_lud}
Consider the LUD problem after SDR:
\begin{equation}
\min_{G\succcurlyeq0}\sum_{i<j}\left\Vert \vec{c}_{ij}^{T}-G_{ij}\vec{c}_{ji}^{T}\right\Vert \text{ s.t. }\mathcal{A}\left(G\right)={\bf b},\label{eq:L1_convex}
\end{equation}
where $G$, $\mathcal{A}$ and $\bf b$ are defined in (\ref{eq:G}) and (\ref{eq:Ab}) respectively. The ADMM devised to solve (\ref{eq:L1_convex}) is similar to and simpler than the ADMM devised to solve the one with the spectral norm constraint. We focus on the more difficult problem with the spectral norm constraint. Introducing ${\bf x}_{ij}=\vec{c}_{ij}^{T}-G_{ij}\vec{c}_{ji}^{T}$
and adding the spectral norm constraint $\left\Vert G\right\Vert _{2}\leq\alpha K$,
we obtain
\begin{equation}
\min_{{\bf x}_{ij},G\succcurlyeq0}\sum_{i<j}\left\Vert {\bf x}_{ij}\right\Vert \text{ s.t. }\mathcal{A}\left(G\right)={\bf b},\,{\bf x}_{ij}=\vec{c}_{ij}^{T}-G_{ij}\vec{c}_{ji}^{T},\,\left\Vert G\right\Vert _{2}\leq\alpha K.\label{eq:L1_ADM_primal}
\end{equation}
The dual problem of problem (\ref{eq:L1_ADM_primal}) is
\begin{equation}
\max_{\boldsymbol{\theta}_{ij},{\bf y},X\succcurlyeq0}\min_{{\bf x}_{ij},\left\Vert G\right\Vert _{2}\leq\alpha K}\sum_{i<j}\left(\left\Vert {\bf x}_{ij}\right\Vert -\left\langle \boldsymbol{\theta}_{ij},{\bf x}_{ij}-\vec{c}_{ij}^{T}+G_{ij}\vec{c}_{ji}^{T}\right\rangle \right)-\left\langle {\bf y},\mathcal{A}\left(G\right)-{\bf b}\right\rangle -\left\langle G,X\right\rangle .\label{eq:L1_dual1}
\end{equation}
By rearranging terms in (\ref{eq:L1_dual1}), we obtain
\begin{eqnarray}
\max_{\boldsymbol{\theta}_{ij},{\bf y},X\succcurlyeq0}\min_{{\bf x}_{ij},\left\Vert G\right\Vert _{2}\leq\alpha K}&-&\left\langle Q\left(\boldsymbol{\theta}\right)+X+\mathcal{A}^{*}\left({\bf y}\right),G\right\rangle +{\bf y}^{T}{\bf b} \nonumber\\
&+&\sum_{i<j}\left(\left\Vert {\bf x}_{ij}\right\Vert -\left\langle \boldsymbol{\theta}_{ij},{\bf x}_{ij}\right\rangle +\left\langle \boldsymbol{\theta}_{ij},\vec{c}_{ij}^{T}\right\rangle \right),\label{eq:L1_dual2}
\end{eqnarray}
where $\boldsymbol{\theta}=\left(\boldsymbol{\theta}_{ij}\right)_{i,j=1,\ldots,K}$,
$\boldsymbol{\theta}_{ij}=\left(\theta_{ij}^{1},\theta_{ij}^{2}\right)^{T}$,
$\vec{c}_{ij}=\left(c_{ij}^{1},c_{ij}^{2}\right)$,
\[
Q\left(\boldsymbol{\theta}\right)=\frac{1}{2}\left(\begin{array}{cc}
Q^{11}\left(\boldsymbol{\theta}\right) & Q^{12}\left(\boldsymbol{\theta}\right)\\
Q^{21}\left(\boldsymbol{\theta}\right) & Q^{22}\left(\boldsymbol{\theta}\right)
\end{array}\right)\text{ and }Q^{pq}\left(\boldsymbol{\theta}\right)=\left(\begin{array}{cccc}
0 & \theta_{12}^{p}c_{21}^{q} & \cdots & \theta_{1K}^{p}c_{K1}^{q}\\
c_{21}^{q}\theta_{12}^{p} & 0 & \cdots & \theta_{2K}^{p}c_{K2}^{q}\\
\vdots & \vdots & \ddots & \vdots\\
c_{K1}^{q}\theta_{1K}^{p} & c_{K2}^{q}\theta_{2K}^{p} & \cdots & 0
\end{array}\right)
\]
for $p$, $q=1,2$. It is easy to verify that for $1\leq i<j\leq K$
\begin{equation}
\min_{{\bf x}_{ij}}\left(\left\Vert {\bf x}_{ij}\right\Vert -\left\langle \boldsymbol{\theta}_{ij},{\bf x}_{ij}\right\rangle \right)=\begin{cases}
0 & \text{ if \ensuremath{\left\Vert \boldsymbol{\theta}_{ij}\right\Vert \leq}1}\\
-\infty & \text{otherwise.}
\end{cases}\label{eq:theta_less_1}
\end{equation}
In fact, (\ref{eq:theta_less_1}) is obtained using the inequality
\begin{eqnarray}
\left\Vert {\bf x}_{ij}\right\Vert -\left\langle \boldsymbol{\theta}_{ij},{\bf x}_{ij}\right\rangle  & = & \left\Vert {\bf x}_{ij}\right\Vert -\left\Vert \boldsymbol{\theta}_{ij}\right\Vert \left\Vert \boldsymbol{x}_{ij}\right\Vert \left\langle \boldsymbol{\theta}_{ij}/\left\Vert \boldsymbol{\theta}_{ij}\right\Vert ,{\bf x}_{ij}/\left\Vert \boldsymbol{x}_{ij}\right\Vert \right\rangle \nonumber \\
 & \geq & \left\Vert {\bf x}_{ij}\right\Vert -\left\Vert \boldsymbol{\theta}_{ij}\right\Vert \left\Vert \boldsymbol{x}_{ij}\right\Vert =\left(1-\left\Vert \boldsymbol{\theta}_{ij}\right\Vert \right)\left\Vert \boldsymbol{x}_{ij}\right\Vert ,\label{eq:theta_less_1_ineq}
\end{eqnarray}
and the inequality in (\ref{eq:theta_less_1_ineq} ) holds when $\boldsymbol{\theta_{ij}}$
and ${\bf x}_{ij}$ have the same direction. Using the fact that the
dual norm of the spectral norm is the nuclear norm and the fact in (\ref{eq:theta_less_1}),
we can obtain from (\ref{eq:L1_dual2}) the dual problem
\begin{eqnarray}
\min_{\boldsymbol{\theta}_{ij},{\bf y},X\succcurlyeq0} & -{\bf y}^{T}{\bf b}-\sum_{i<j}\left\langle \boldsymbol{\theta}_{ij},\vec{c}_{ij}^{T}\right\rangle +\alpha K\left\Vert Z\right\Vert _{*}\label{eq:L1_dual_3}\\
\text{s.t.} & Z=Q\left(\boldsymbol{\theta}\right)+X+\mathcal{A}^{*}\left({\bf y}\right), & \text{and }\left\Vert \boldsymbol{\theta}_{ij}\right\Vert \leq1.\label{eq:L1_dual_4}
\end{eqnarray}
The augmented Lagrangian function of problem (\ref{eq:L1_dual_3})-(\ref{eq:L1_dual_4})
is defined as
\begin{eqnarray}
\mathcal{L}\left({\bf y},\boldsymbol{\theta},Z,X,G\right)=&-&{\bf y}^{T}{\bf b}+\alpha K\left\Vert Z\right\Vert _{*}-\sum_{i<j}\left\langle \boldsymbol{\theta}_{ij},\vec{c}_{ij}^{T}\right\rangle +\left\langle G,Q\left(\boldsymbol{\theta}\right)+X+\mathcal{A}^{*}\left({\bf y}\right)-Z\right\rangle\nonumber \\
&+&\frac{\mu}{2}\left\Vert Q\left(\boldsymbol{\theta}\right)+X+\mathcal{A}^{*}\left({\bf y}\right)-Z\right\Vert _{F}^{2},\label{eq:L1_Lagrangian}
\end{eqnarray}
for $\left\Vert \boldsymbol{\theta}_{ij}\right\Vert \leq1$, where
$\mu>0$ is a penalty parameter. Similar to section \ref{sec:adm_wls}, using the
augmented Lagrangian function (\ref{eq:L1_Lagrangian}), ADMM is used
to minimize (\ref{eq:L1_Lagrangian}) with respect to ${\bf y}$,
$\boldsymbol{\theta}$, $Z$, $X$, and $G$ alternatively, that is,
given some initial guess, in each iteration the following four subproblems
are solved sequentially:
\begin{eqnarray}
{\bf y}^{k+1} & = & \arg\min_{{\bf y}}\mathcal{L}\left({\bf y},\boldsymbol{\theta}^{k},Z^{k},X^{k},G^{k}\right),\label{eq:L1_sub_y}\\
\boldsymbol{\theta}_{ij}^{k+1} & = & \arg\min_{{\bf \left\Vert \boldsymbol{\theta}_{ij}\right\Vert \leq1}}\mathcal{L}\left({\bf y}^{k+1},\boldsymbol{\theta},Z^{k},X^{k},G^{k}\right),\label{eq:L1_sub_theta}\\
Z^{k+1} & = & \arg\min_{Z}\mathcal{L}\left({\bf y}^{k+1},\boldsymbol{\theta}^{k+1},Z,X^{k},G^{k}\right),\label{eq:L1_sub_Z}\\
X^{k+1} & = & \arg\min_{X\succcurlyeq0}\mathcal{L}\left({\bf y}^{k+1},\boldsymbol{\theta}^{k+1},Z^{k+1},X,G^{k}\right),\label{eq:L1_sub_X}
\end{eqnarray}
and the Lagrange multiplier $G$ is updated by
\begin{equation}
G^{k+1}=G^{k}+\gamma\mu\left(Q\left(\boldsymbol{\theta}^{k+1}\right)+X^{k+1}+\mathcal{A}^{*}\left({\bf y}^{k+1}\right)-Z^{k+1}\right),\label{eq:L1_sub_theta_1}
\end{equation}
where $\gamma\in\left(0,\frac{1+\sqrt{5}}{2}\right)$ is an approprately
chosen step length. The methods to solve subproblems (\ref{eq:L1_sub_y}),
(\ref{eq:L1_sub_Z}) and (\ref{eq:L1_sub_X}) are similar to those
used in (\ref{eq:sub_y}),
(\ref{eq:sub_Z}) and (\ref{eq:sub_X}). To solve subproblem (\ref{eq:L1_sub_theta}),
we rearrange the terms of $\mathcal{L}\left({\bf y}^{k+1},\boldsymbol{\theta},Z^{k},X^{k},G^{k}\right)$
and obtain an eqivalent problem
\[
\min_{\boldsymbol{\theta}_{ij}}-\left\langle \boldsymbol{\theta}_{ij},\vec{c}_{ij}^{T}\right\rangle +\frac{\mu}{2}\left\Vert \boldsymbol{\theta}_{ij}\vec{c}_{ji}+\Phi_{ij}\right\Vert _{F}^{2},\text{ s.t. }\left\Vert \boldsymbol{\theta}_{ij}\right\Vert \leq1,
\]
where $\Phi=X^{k}+\mathcal{A}^{*}\left({\bf y}^{k+1}\right)-Z^{k}+\frac{1}{\mu}G^{k}$
, $\Phi=\left(\begin{array}{cc}
\Phi^{11} & \Phi^{12}\\
\Phi^{21} & \Phi^{22}
\end{array}\right)$ and $\Phi_{ij}=\left(\begin{array}{cc}
\Phi_{ij}^{11} & \Phi_{ij}^{12}\\
\Phi_{ij}^{21} & \Phi_{ij}^{22}
\end{array}\right)$. Problem (\ref{eq:L1_sub_theta_1}) is further simplified as
\[
\min_{\boldsymbol{\theta}_{ij}}\left\langle \boldsymbol{\theta}_{ij},\mu\Phi_{ij}\vec{c}_{ji}^{T}-\vec{c}_{ij}^{T}\right\rangle +\frac{\mu}{2}\left\Vert \boldsymbol{\theta}_{ij}\right\Vert ^{2},\text{ s.t. }\left\Vert \boldsymbol{\theta}_{ij}\right\Vert \leq1,
\]
whose solution is
\[
\boldsymbol{\theta}_{ij}=\begin{cases}
\frac{1}{\mu}\vec{c}_{ij}^{T}-\Phi_{ij}\vec{c}_{ij}^{T}&\text{if } \left\Vert\frac{1}{\mu}\vec{c}_{ij}^{T}-\Phi_{ij}\vec{c}_{ij}^{T}\right\Vert\leq1,\\
\frac{\vec{c}_{ij}^{T}-\mu\Phi_{ij}\vec{c}_{ij}^{T}}{\left\Vert \vec{c}_{ij}^{T}-\mu\Phi_{ij}\vec{c}_{ij}^{T}\right\Vert }& \text{ otherwise.}
\end{cases}
\]

The practical issues related to how to
take advantage of low-rank assumption of $G$ in the eigenvalue decomposition
performed at each iteration, strategies for adjusting the penalty
parameter $\mu$, the use of a step size $\gamma$ for updating the
primal variable $X$ and termination rules using the in-feasibility
measures are discussed in details in \cite{ADM}.
The convergence analysis on ADMM using more than two blocks of variables can be found
in \cite{HongLuo2012}. However, there is one condition of Assumption A (page 5) in \cite{HongLuo2012} that cannot be satisfied for our problem: the condition that the feasible set should be polyhedral, whereas the SDP cone in our problem is not a polyhedral. To generalize the convergence analysis in \cite{HongLuo2012} to our problem, we will need to show that the local error bounds (page 8 - 9 in \cite{HongLuo2012}) hold for the SDP cone. Currently we do not  have a rigorous convergence proof for ADMM for our problem.

\section{The Iterative Reweighted Least Squares (IRLS) Procedure}
\label{sec:IRLS}

Since $\vec{c}_{ij}$ and $\vec{c}_{ji}$ are unit vectors, it is tempting to replace the LUD problem (\ref{eq:lad-exact_form2}) with the following semidefinite relaxation:
\begin{eqnarray}
 & \min_{G\in\mathbb{R}^{2K\times2K}} & F(G) =\sum_{i,j=1,2,\ldots,K}\sqrt{2-2\sum_{p,q=1,2}G_{ij}^{pq}S_{ij}^{pq}}\label{eq:sdp_algorithm1}\\
& \text{s.t. }& G_{ii} = I_2,\text{ }i=1,2,\cdots,K,\label{eq:sdp_orth_constraint_algorithm1}\\
 &  & G\succcurlyeq0,\label{eq:algorithm1}\\
& & \left\Vert G\right\Vert _{2}\leq\alpha K \text{ (optional),}\label{eq:optional_alpha}
\end{eqnarray}
where $\alpha$ is a fixed number between $\frac{2}{3}$ and $1$, and the spectral
norm constraint on $G$ (\ref{eq:optional_alpha}) is added when the solution to
the problem (\ref{eq:sdp_algorithm1})-(\ref{eq:algorithm1}) is a set of highly
clustered rotations. Notice that this relaxed problem is, however, not convex
since the objective function (\ref{eq:sdp_algorithm1}) is concave. We propose
 to solve (\ref{eq:sdp_algorithm1})-(\ref{eq:algorithm1}) (possibly with
(\ref{eq:optional_alpha})) by an variant of the IRLS procedure
\cite{irls_d,irls_t,CandesWakin-Boyd2008}, which at best converges to a local minimizer. With a good initial guess for $G$ it can be hoped that the global minimizer is obtained. Such an initial guess can be taken as the LS solution.

\begin{algorithm}
\caption{\label{alg:Algorithm_1} ({\bf the IRLS procedure}) Solve optimization problem (\ref{eq:sdp_algorithm1})-(\ref{eq:algorithm1}) (with the spectral norm constraint on $G$ (\ref{eq:optional_alpha}) if the input parameter $\alpha$ satisfies $\frac{2}{3}\leq \alpha < 1$), and then recover the orientations by rounding.}
\begin{algorithmic}
\Require a $2K\times2K$ common-line matrix $S$, a regularization parameter $\epsilon$, a parameter $\alpha$ and the total number of iterations $N_{\text{iter}}$
   \State $w_{ij}^0=1$ $\forall i,j = 1,\cdots,K$;
   \State $G^0=0$;
   \For{$k = 1 \to N_{\text{iter}}$, step size = 1}
   \State  update $W$ by setting $w_{ij} = w_{ij}^{k-1}$;
   \State if $\frac{2}{3}\leq \alpha < 1$, obtain $G^k$ by solving the problem (\ref{eq:adm_obj})-(\ref{eq:adm_spectral}) using ADMM; otherwise, obtain $G^k$ by solving (\ref{eq:sdp})\textendash{}(\ref{eq:sdp_constraint}) using SDPLR (with initial guess $G^{k-1}$);
  \State $r_{ij}^k=\sqrt{2-2\sum_{p,q=1,2}G_{ij}^{pq}S_{ij}^{pq} + \epsilon^2}$;
  \State $w^k_{ij}=1/ r_{ij}^k$;
  \State the residual $r^k=\sum_{i,j=1}^K r_{ij}^k$;
\EndFor
\State obtain estimated orientations $\hat{R}_1,\ldots,\hat{R}_K$ from $G^{N_{\text{iter}}}$ using the randomized rounding procedure in section \ref{sec:rounding}.
\end{algorithmic}
\end{algorithm}
Before the rounding procedure, the IRLS procedure finds an approximate solution to the optimization problem (\ref{eq:sdp_algorithm1})-(\ref{eq:algorithm1}) (possibly with (\ref{eq:optional_alpha})) by solving its smoothing version
\begin{eqnarray}
 & \min_{G\in\mathbb{R}^{2K\times2K}} & F(G,\epsilon)=\sum_{i,j=1,2,\ldots,K}\sqrt{2-2\sum_{p,q=1,2}G_{ij}^{pq}S_{ij}^{pq} + \epsilon^2}\label{eq:irls_epsilon_obj}\\
&\text{s.t. } & G_{ii} = I_2,\text{ }i=1,2,\cdots,K,\label{eq:sdp_orth_constraint_irls}\\
 &  & G\succcurlyeq0,\label{eq:irls_sdp}\\
& & \left\Vert G\right\Vert _{2}\leq\alpha K \text{ (optional).}\label{eq:irls_optional_alpha}
\end{eqnarray}
where $\epsilon>0$ is a small number. The solution to the smoothing version is close to the solution to the original problem. In fact, let $G^*_{\epsilon} = \arg\min F(G,\epsilon)$ and $G^* = \arg\min F(G)$, then we shall verify that
\begin{equation}
\left| F(G^*_{\epsilon}) - F(G^*) \right| \leq 4K^2 \epsilon.
\label{eq:closeness}
\end{equation}
Using the fact that 
\[
0\leq F(G,\epsilon)-F(G)<4K^{2}\epsilon,
\]
we obtain 
\begin{eqnarray*}
 & (F(G^{*},\epsilon)-F(G_{\epsilon}^{*},\epsilon))+(F(G_{\epsilon}^{*})-F(G^{*}))\\
= & (F(G^{*},\epsilon)-F(G^{*}))-(F(G_{\epsilon}^{*},\epsilon)-F(G_{\epsilon}^{*})) & \text{\ensuremath{\leq}}4K^{2}\epsilon.
\end{eqnarray*}
Since $F(G^{*},\epsilon)-F(G_{\epsilon}^{*},\epsilon)\geq0$ and $F(G_{\epsilon}^{*})-F(G^{*})\geq0$, the inequality (\ref{eq:closeness}) holds.

In each iteration, the IRLS procedure solves the problem
\begin{equation}
G^{k+1}=\arg\min_{G\succcurlyeq0}\sum_{i\neq j}w_{ij}^{k}\left(2-2\left\langle G_{ij},S_{ij}\right\rangle +\epsilon^{2}\right)\text{ s.t. }\mathcal{A}(G)={\bf b},(\text{optional:} \left\Vert G\right\Vert _{2}\leq\alpha K  )\label{eq:irls_1}
\end{equation}
on the $(k+1)$th iteration, where $w_{ij}^{0}=1$, and
\[
w_{ij}^{k}=1/\sqrt{2-2\left\langle G_{ij}^{k},S_{ij}\right\rangle +\epsilon^{2}},\:\forall k>0.
\]
In other words, in each iteration, more emphasis is given to detected common-lines that are better explained
by the current estimate $G^k$ of the Gram matrix. The inclusion of the regularization parameter
$\epsilon$ ensures that no single detected common-line can gain undue
influence when solving
\begin{equation}
G^{k+1}=\arg\max_{G\succcurlyeq0}\left\langle W^{k}\circ S,G\right\rangle \text{ s.t. }\mathcal{A}(G)={\bf b} \text{ (optional:} \left\Vert G\right\Vert _{2}\leq\alpha K  ).\label{eq:irls_2}
\end{equation}
We repeat the process until the residual sequence $\{r^k\}$ has converged, or the maximum number of iterations has been reached. We shall verify that the value of the cost function is non-increasing, and that every cluster point of the sequence of IRLS is a stationary point of (\ref{eq:irls_epsilon_obj}) - (\ref{eq:irls_sdp}) in the following lemma and theorem, for the problem without the spectral norm constraint on $G$. The arguments can be generalized to the case with
the spectral norm constraint. The proof of Theorem \ref{thm:convergence} follows the method of proof for Theorem 3  in the paper \cite{Mohan2012} by Mohan et. al..

\begin{lemma}\label{lem:F_decreasing}
The value of the cost function sequence is monotonically non-increasing, i.e., 
\begin{equation}
F(G^{k+1},\epsilon)\leq F(G^{k},\epsilon).\label{eq:ineqn_1}
\end{equation}
where $\left\{ G^{k}\right\} $ is the sequence generated by the IRLS procedure of Algorithm \ref{alg:Algorithm_1}.
\end{lemma}

\begin{proof}
Since $G^{k}$ is the solution of (\ref{eq:irls_2}), there exists
${\bf y}^{k}\in\mathbb{R}^{2K}$ and $X^{k}\in\mathbb{R}^{2K\times2K}$
such that
\begin{gather}
-\mathcal{A}^{*}({\bf y}^{k})+X^{k}+W^{k-1}\circ S=0,\:\mathcal{A}(G^{k})-{\bf b}=0,\label{eq:eqn_y_X_G}\\
G^{k}\succcurlyeq0,\: X^{k}\succcurlyeq0,\:\left\langle G^{k},X^{k}\right\rangle =0.\label{eq:PSD_X_G}
\end{gather}
Hence we have
\begin{eqnarray}
0 & = & -({\bf y}^{k})^{T}(\mathcal{A}(G^{k})-{\bf b})+({\bf y}^{k+1})^{T}(\mathcal{A}(G^{k+1})-{\bf b})\nonumber \\
 & = & ({\bf y}^{k+1}-{\bf y}^{k})^{T}(\mathcal{A}(G^{k})-{\bf b})+\left\langle \mathcal{A}^{*}({\bf y}^{k+1}),G^{k+1}-G^{k}\right\rangle \nonumber \\
 & = & \left\langle X^{k+1}+W^{k}\circ S,G^{k+1}-G^{k}\right\rangle \nonumber \\
 & \leq & \left\langle W^{k}\circ S,G^{k+1}-G^{k}\right\rangle \label{eq:ineqn_F}\\
 & = & \frac{1}{2}\sum_{i\neq j}\left(-\beta_{ij}^{k}\left(2-2\left\langle G_{ij}^{k+1},S_{ij}\right\rangle +\epsilon^{2}\right)+\beta_{ij}^{k}\left(2-2\left\langle G_{ij}^{k},S_{ij}\right\rangle +\epsilon^{2}\right)\right)\nonumber \\
 & = & \frac{1}{2}\sum_{i\neq j}\left(-\frac{2-2\left\langle G_{ij}^{k+1},S_{ij}\right\rangle +\epsilon^{2}}{\sqrt{2-2\left\langle G_{ij}^{k},S_{ij}\right\rangle +\epsilon^{2}}}+\sqrt{2-2\left\langle G_{ij}^{k},S_{ij}\right\rangle +\epsilon^{2}}\right),\label{eq:ineqn_2}
\end{eqnarray}
where the third equality uses (\ref{eq:eqn_y_X_G}), and the inequality
(\ref{eq:ineqn_F}) uses (\ref{eq:PSD_X_G}). From (\ref{eq:ineqn_2})
we obtain
\begin{align}
F(G^{k},\epsilon)^{2} & =\left(\sum_{i\neq j}\sqrt{2-2\left\langle G_{ij}^{k},S_{ij}\right\rangle +\epsilon^{2}}\right)^{2}\nonumber \\
 & \geq\left(\sum_{i\neq j}\sqrt{2-2\left\langle G_{ij}^{k},S_{ij}\right\rangle +\epsilon^{2}}\right)\left(\sum_{i\neq j}\frac{2-2\left\langle G_{ij}^{k+1},S_{ij}\right\rangle +\epsilon^{2}}{\sqrt{2-2\left\langle G_{ij}^{k},S_{ij}\right\rangle +\epsilon^{2}}}\right)\nonumber \\
 & \geq\left(\sum_{i\neq j}\sqrt{2-2\left\langle G_{ij}^{k+1},S_{ij}\right\rangle +\epsilon^{2}}\right)^{2}=F(G^{k+1},\epsilon)^{2},\label{eq:ineqn_3}
\end{align}
where the last inequality uses Cauchy-Schwarz inequality and the equality
holds if and only if
\begin{equation}
\frac{\sqrt{2-2\left\langle G_{ij}^{k+1},S_{ij}\right\rangle +\epsilon^{2}}}{\sqrt{2-2\left\langle G_{ij}^{k},S_{ij}\right\rangle +\epsilon^{2}}}=c\text{ for all }i\neq j,\label{eq:eq_condition}
\end{equation}
where $c$ is a constant. Thus (\ref{eq:ineqn_1}) is confirmed.
\end{proof}

\begin{theorem}
\label{thm:convergence}
The sequence of iterates $\left\{ G^{k}\right\} $ of IRLS is bounded,
and every cluster point of the sequence is a stationary point of (\ref{eq:irls_epsilon_obj}) - (\ref{eq:irls_sdp}).
\end{theorem}

\begin{proof}
Since trace$(G^{k})=2K$ and $G^{k}\succcurlyeq0$, the sequence $\left\{ G^{k}\right\} $
is bounded. It follows that $W^{k}$ and trace$((W^{k}\circ S)G^{k+1})$
are bounded. Using the strong duality of SDP, we conclude that ${\bf b}^{T}{\bf y}^{k+1}=$trace$((W^{k}\circ S)G^{k+1})$
is bounded. In addition, from the KKT conditions (\ref{eq:eqn_y_X_G}) - (\ref{eq:PSD_X_G}) we
obtain $-\mathcal{A}^{*}({\bf y}^{k+1})+W^{k}\circ S\succcurlyeq0$.
Using the definition of $\mathcal{A}^{*}$ and $S$, the property
of semi-definite matrices and the fact that $W^{k}$ is bounded, it
can be verified that ${\bf y}^{k}$ is bounded. Using (\ref{eq:eqn_y_X_G}) again,
we obtain
\[
\left\Vert X^{k}\right\Vert =\left\Vert \mathcal{A}^{*}({\bf y}^{k})-W^{k-1}\circ S\right\Vert \leq\left\Vert \mathcal{A}^{*}({\bf y}^{k})\right\Vert +\left\Vert W^{k-1}\circ S\right\Vert ,
\]
which implies that $X^{k}$ is bounded.

We now show that every cluster point of $\left\{ G^{k}\right\} $
is a stationary point of (\ref{eq:irls_epsilon_obj}) - (\ref{eq:irls_sdp}). Suppose to the contrary and
let $\bar{G}$ be a cluster point of $\left\{ G^{k}\right\} $ that
is not a stationary point. By the definition of cluster point, there
exists a subsequence $\left\{ G^{k_{i}},W^{k_{i}},X^{k_{i}},{\bf y}^{k_{i}}\right\} $
of $\left\{ G^{k},W^{k},X^{k},{\bf y}^{k}\right\} $ converging to
$\left(\bar{G},\bar{W},\bar{X},\bar{y}\right)$. By passing to a further
subsequence if necessary, we can assume that $\left\{ G^{k_{i}+1},W^{k_{i}+1},X^{k_{i}+1},{\bf y}^{k_{i}+1}\right\} $
is also convergent and we denote its limit by $\left(\hat{G},\hat{W},\hat{X},\hat{{\bf y}}\right)$.
$G^{k_{i}+1}$ is defined as (\ref{eq:irls_1}) or (\ref{eq:irls_2}) and satisfies the KKT
conditions (\ref{eq:eqn_y_X_G}) - (\ref{eq:PSD_X_G}). Passing to limits, we see that 

\begin{align*}
-\mathcal{A}^{*}(\hat{{\bf y}})+\hat{X}+\bar{W}\circ S=0, & \mathcal{A}(\hat{G})-{\bf b}=0,\\
\hat{G}\succcurlyeq0,\hat{X}\succcurlyeq0, & \left\langle \hat{G},\hat{X}\right\rangle =0.
\end{align*}
Thus we conclude that $\hat{G}$ is a maximizer of the following convex
optimization problem,
\[
\max_{G\succcurlyeq0}\left\langle \bar{W}\circ S,G\right\rangle \text{ s.t. }\mathcal{A}(\hat{G})={\bf b}.
\]
Next, by assumption, $\bar{G}$ is not a stationary point of (\ref{eq:irls_epsilon_obj}) - (\ref{eq:irls_sdp}). This implies that $\bar{G}$ is not a maximizer of the problem
above and thus $\left\langle \bar{W}\circ S,\hat{G}\right\rangle >\left\langle \bar{W}\circ S,\bar{G}\right\rangle $.
From this last relation and (\ref{eq:ineqn_3}) - (\ref{eq:eq_condition}) it follows that 
\begin{equation}
F(\hat{G},\epsilon)<F(\bar{G},\epsilon).\label{eq:contradiction}
\end{equation}
Otherwise if $F(\hat{G},\epsilon)=F(\bar{G},\epsilon)$, then $\left\langle \hat{G}_{ij},S_{ij}\right\rangle =\left\langle \bar{G}_{ij},S_{ij}\right\rangle $
due to (\ref{eq:ineqn_3}) - (\ref{eq:eq_condition}), and thus we would obtain $\left\langle \bar{W}\circ S,\hat{G}\right\rangle =\left\langle \bar{W}\circ S,\bar{G}\right\rangle $
which is a contradiction.

On the other hand, it follows from Lemma \ref{lem:F_decreasing} that the sequence $\left\{ F(G^{k},\epsilon)\right\} $
converges. Thus we have that
\[
\lim F(G^{k},\epsilon)=\lim F(G^{k_{i}},\epsilon)=F(\bar{G},\epsilon)=\lim F(G^{k_{i}+1},\epsilon)=F(\hat{G},\epsilon)
\]
which contradicts (\ref{eq:contradiction}). Hence, every cluster
point of the sequence is a stationary point of (\ref{eq:irls_epsilon_obj}) - (\ref{eq:irls_sdp}).
\end{proof}

In addition, using H\"{o}lder's inequality, the analysis can be generalized to the reweighted approach to solve
\begin{equation}
\min_{G\succcurlyeq0}\sum_{i\neq j}\left(2-2\left\langle G_{ij},S_{ij}\right\rangle \right)^{\frac{p}{2}}\text{ s.t. }\mathcal{A}(G)={\bf b},(\text{optional:} \left\Vert G\right\Vert _{2}\leq\alpha K  )\label{eq:irls_p}
\end{equation}
where $0 < p < 1$. Convergence analysis of IRLS for different applications with $p<1$ can be found in
\cite{irls_d,irls_t}. The problem (\ref{eq:irls_p}) is a SDR of the problem
\begin{equation}
\min_{R_{1},\ldots,R_{K}\in{\bf SO}\left(3\right)}\sum_{i\neq j}\left\Vert R_{i}\left(\vec{c}_{ij},0\right)^{T}-R_{j}\left(\vec{c}_{ji},0\right)^{T}\right\Vert^p.
\end{equation}
The smaller $p$ is, the more penalty the outliers in the detected common-lines
receive. 
\section{Numerical results}
\label{sec:numerical_results}

All numerical experiments were performed on a machine with 2 Intel(R)
Xeon(R) CPUs X5570, each with 4 cores, running at 2.93 GHz. In all the experiments, the polar Fourier transform of images for common-line detection had radial resolution $n_r =100$ and angular resolution $n_\theta=360$. The number of iterations was set to be $N_{\text{iter}} = 10$ in all IRLS procedures. The reconstruction from the images with estimated orientations used the Fourier based 3D reconstruction package FIRM\footnote{The FIRM package is available at \url{https://web.math.princeton.edu/~lanhuiw/software.html}.} \cite{FIRM}. The reconstructed volumes are shown in Figure \ref{fig:reconstructions} and \ref{fig:reconstruction_real} using the visualization system Chimera \cite{Chimera}.

To evaluate the accuracy or the resolution of the reconstructions, we used  the 3D Fourier
Shell Correlation (FSC) \cite{FSC}. FSC measures the normalized cross-correlation
coefficient between two 3D volumes over corresponding spherical shells
in Fourier space, i.e.,
\begin{equation}
\text{FSC}\left(i\right)=\frac{\sum_{{\bf j} \in Shell_i}{\mathcal F}\left({\bf V}_{1}\right)\left({\bf j}\right)\cdot \overline{{\mathcal F}\left({\bf V}_{2}\right)\left({\bf j}\right)}}{\sqrt{\sum_{{\bf j} \in Shell_i}\left|{\mathcal F}\left({\bf V}_{1}\right)\left({\bf j}\right)\right|^{2}\cdot\sum_{{\bf j} \in Shell_i}\left|{\mathcal F}\left({\bf V}_{2}\right)\left({\bf j}\right)\right|^{2}}},\label{eq:fsc}
\end{equation}
where ${\mathcal F}\left({\bf V}_{1}\right)$ and ${\mathcal F}\left({\bf V}_{2}\right)$ are the Fourier transforms of volume ${\bf V}_1$
and volume ${\bf V}_2$ respectively,  the spatial frequency $i$ ranges from $1$
to $N/2-1$ times the unit frequency $1/(N\cdot \text{pixel size})$, $N$ is the size of a volume, and $ Shell_i:=\{{\bf j}:0.5+(i-1)+\epsilon\leq\left\Vert {\bf j}\right\Vert < 0.5 + i +\epsilon\}$ where $\epsilon =$1e-4. In this form, the FSC takes two 3D volumes
and converts them into a 1D array. In Section \ref{sec:real_data}, we used the FSC 0.143 cutoff criterion \cite{FSC_cutoff_1, FSC_cutoff_2} to determine the resolutions of the ab-initio models and the refined models.

\begin{figure}
\begin{centering}
\includegraphics[width=0.5\paperwidth]{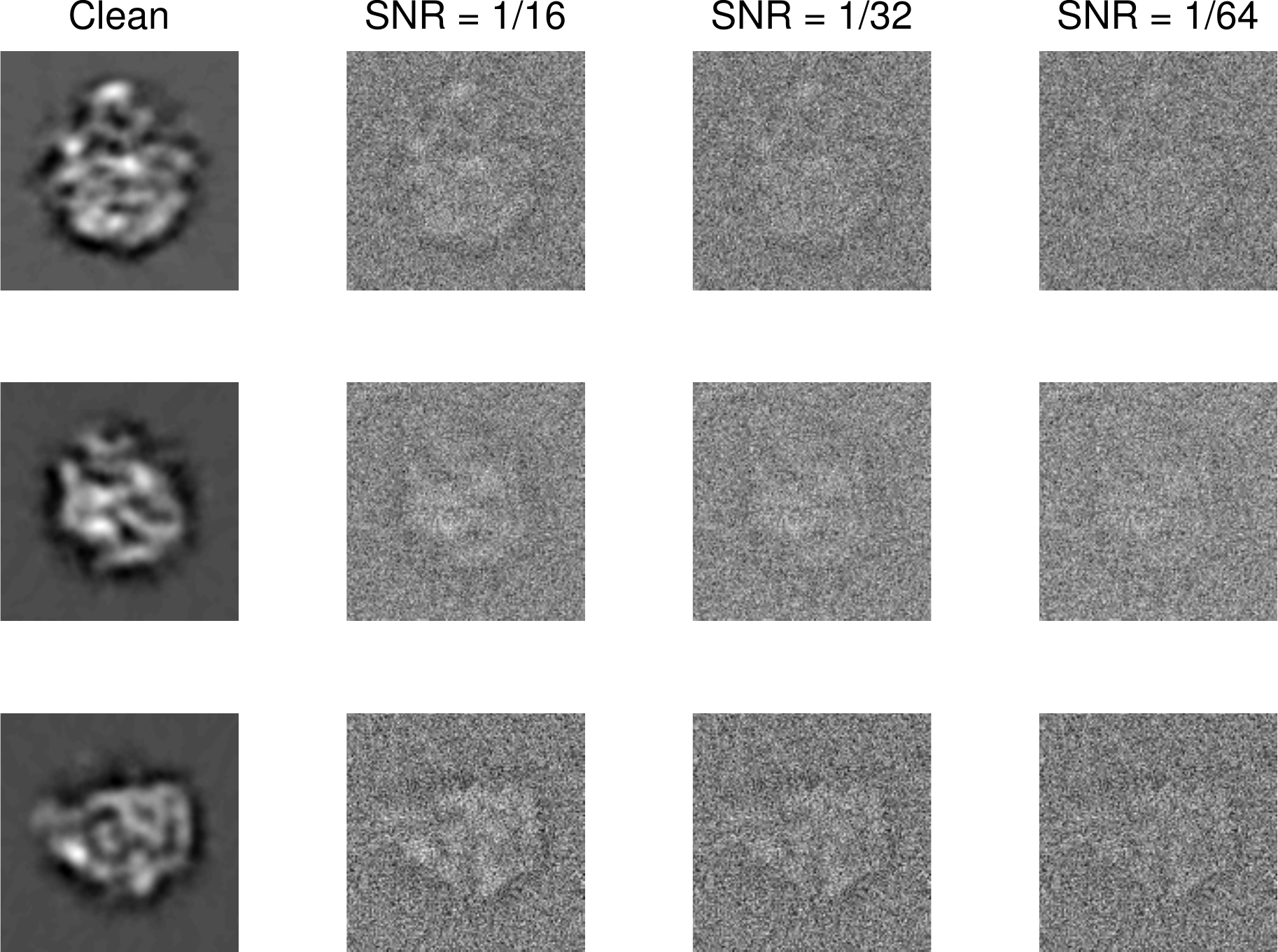}
\par\end{centering}
\caption{\label{fig:images}The first column shows three clean images of size $129 \times 129$ pixels generated from a 50S ribosomal subunit volume with different orientations. The other three columns show three noisy images corresponding to those in the first column with SNR= 1/16, 1/32 and 1/64, respectively.}
\end{figure}

\subsection{Experiments on simulated images}
\begin{figure}
\begin{centering}
\includegraphics[width=0.55\paperwidth]{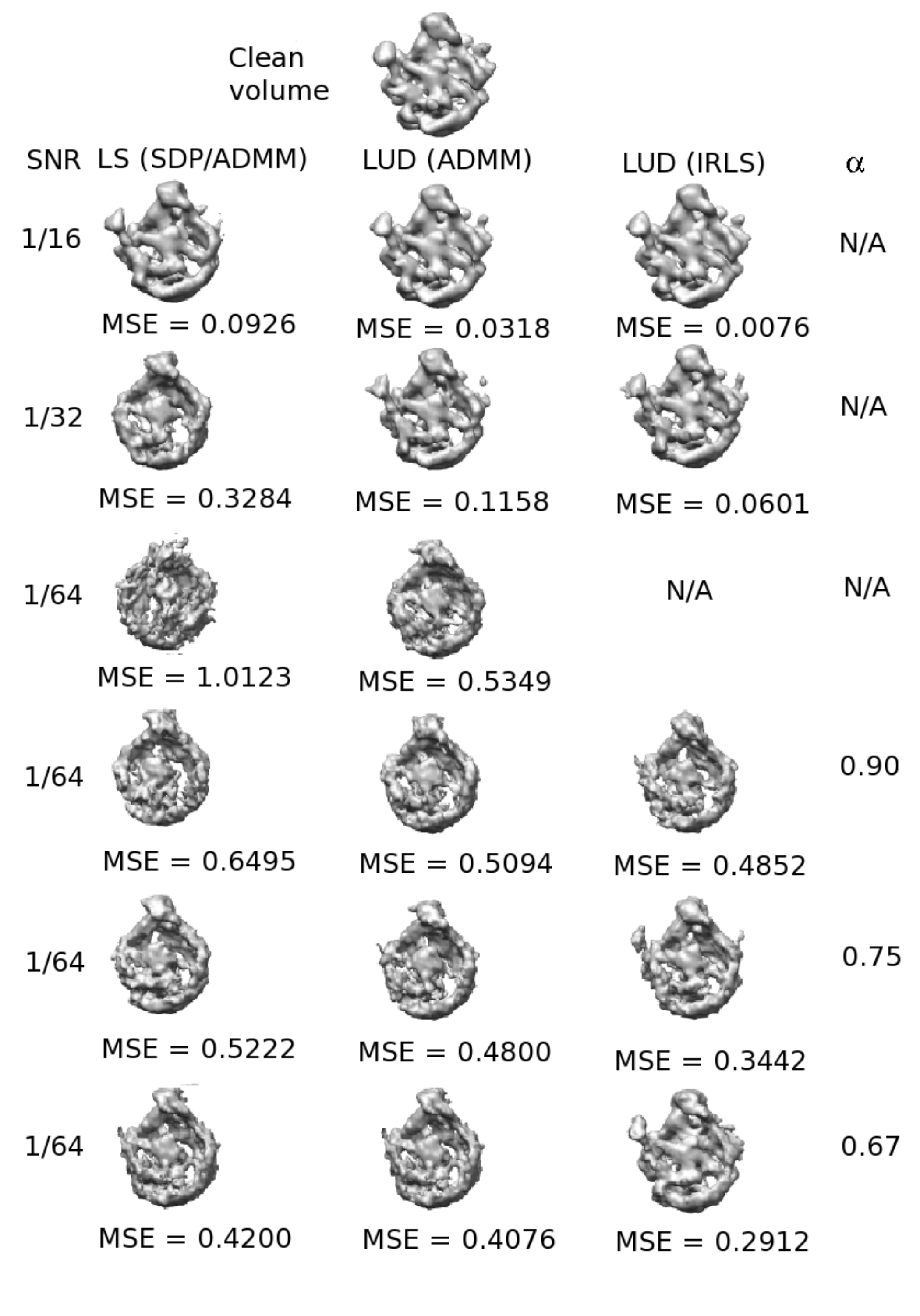}
\par\end{centering}
\caption{\label{fig:reconstructions}The clean volume (top), the reconstructed volumes and the MSEs (\ref{eq:MSE}) of the estimated rotations. From 2nd to 4th row, no spectral norm constraint was used (i.e., $\alpha$ = N/A) for all algorithms. The last 4 rows are all results of very noisy images with SNR = 1/64, where the result using the IRLS procedure without $\alpha$ is not available due to the highly clustered estimated projection directions, and the result from the IRLS procedure with $\alpha = 0.67$ for the spectral norm constraint is best.}
\end{figure}
We simulated $500$ centered images of size $129 \times 129$ pixels with pixel size $2.4$\AA $\text{ }$ of the 50S ribosomal
subunit (the top volume in Figure \ref{fig:reconstructions}), where the orientations of the images were sampled from the uniform
distribution over ${\bf SO}(3)$. White Gaussian noise was added to the clean images to generate noisy images with SNR= 1/16, 1/32 and 1/64 respectively (Figure \ref{fig:images}). Common-line pairs that were detected with an error smaller than $10^{\circ}$ were considered to be correct. The common-line
detection rates were $64\%$, $44\%$ and $23\%$ for images with SNR=1/16, 1/32 and 1/64 respectively (Figure \ref{fig:fat-tails}).

To measure the accuracy of the estimated orientations, we defined the mean squared error (MSE) of the estimated rotation matrices
$\hat{R}_{1},\ldots,\hat{R}_{K}$ as
\begin{equation}
\text{MSE}=\frac{1}{K}\sum_{i=1}^{K}\left\Vert R_{i}-\hat{O}\hat{R}_{i}\right\Vert ^{2},\label{eq:MSE}
\end{equation}
where $\hat{O}$ is the optimal solution to the registration problem
between the two sets of rotations $\left\{ R_{1},\ldots,R_{K}\right\} $
and $\left\{ \hat{R}_{1},\ldots,\hat{R}_{K}\right\} $ in the sense
of minimizing the MSE. As shown in \cite{Amit_eig_sdp}, there is a simple procedure to obtain both $\hat{O}$ and the MSE from the singular value decomposition of the matrix $\frac{1}{K}\sum_{i=1}^{K}\hat{R}_{i}R_{i}^{T}$.

We applied the LS approach using SDP and ADMM, and the LUD approach using ADMM and IRLS to estimate the images' orientations, then computed the MSEs of the estimated rotation matrices, and lastly reconstructed the volume (Figure \ref{fig:reconstructions}).
In order to measure the accuracy of the reconstructed volumes,
we measured each volume's FSC (\ref{eq:fsc}) (Figure \ref{fig:fsc}) against the clean 50S ribosomal subunit volume, that is, in our measurement ${\bf V}_1$ was the reconstructed volume, and ${\bf V}_2$ was the ``ground truth'' volume.

\begin{figure}
\begin{centering}
\includegraphics[width=0.6\paperwidth]{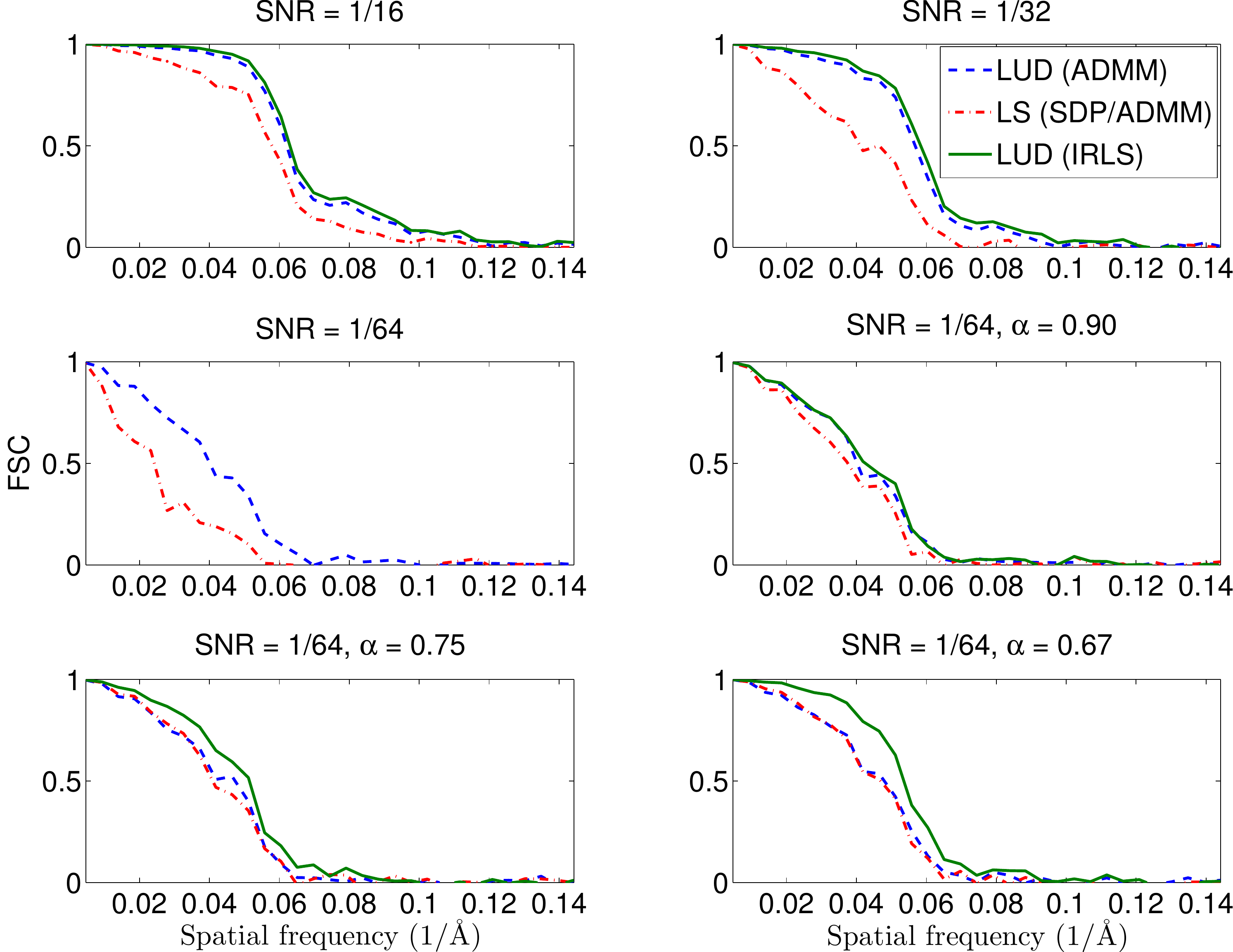}
\par\end{centering}
\caption{\label{fig:fsc}FSCs (\ref{eq:fsc}) of the reconstructed volumes against the clean volume in Figure \ref{fig:reconstructions}. The plots of the correlations show that the LUD approach using ADMM and IRLS (denoted as the blue dashed lines and green solid lines) outweighed the LS approach using SDP or ADMM (denoted as the red dot-dashed lines). Note that all the last four sub-figures are results for images with SNR = 1/64, where the last three sub-figures are results using different $\alpha$ for the spectral norm constraint. In the third sub-figure (left to right, top to bottom), there is no green solid line for the LUD approach using IRLS, since the IRLS procedure without the spectral norm constraint converges to a solution where the estimated projection directions are highly clustered and no 3D reconstruction can be computed.}
\end{figure}
When SNR= 1/16 and 1/32, the common-line detection rate was relatively high (64\%
and 44\%), the algorithms without the spectral  norm constraint on $G$ were enough
to make a good estimation. The LUD approach using ADMM and IRLS outweighed the LS
approach in terms of accuracy measured by MSE and FSC (Figure
\ref{fig:reconstructions}-\ref{fig:fsc}). Note that the LS approach using SDP
failed when SNR = 1 /32, while the LUD approach using either ADMM or IRLS
succeeded. When SNR=1/64, the common-line detection rate was relatively small
(23\%), and most of the detected common-lines were outliers (Figure
\ref{fig:fat-tails}), the algorithms without spectral norm constraint
$\left\Vert G\right\Vert_2 \leq \alpha K$ did not work. Especially, the viewing
directions of images estimated by the IRLS procedures without $\left\Vert
G\right\Vert_2 \leq \alpha K$ converged to two clusters around two antipodal
directions, yielding no 3D reconstruction. The LUD approach using ADMM
failed in this case, however, the IRLS procedure with an appropriate regularization on the spectral norm (i.e., $\alpha = 0.67$ since the true rotations were uniformly sampled over $SO(3)$) gave the best reconstruction.

\subsection{Experiments on a real dataset}
\label{sec:real_data}
\begin{figure}
\begin{centering}
\includegraphics[width=0.6\paperwidth]{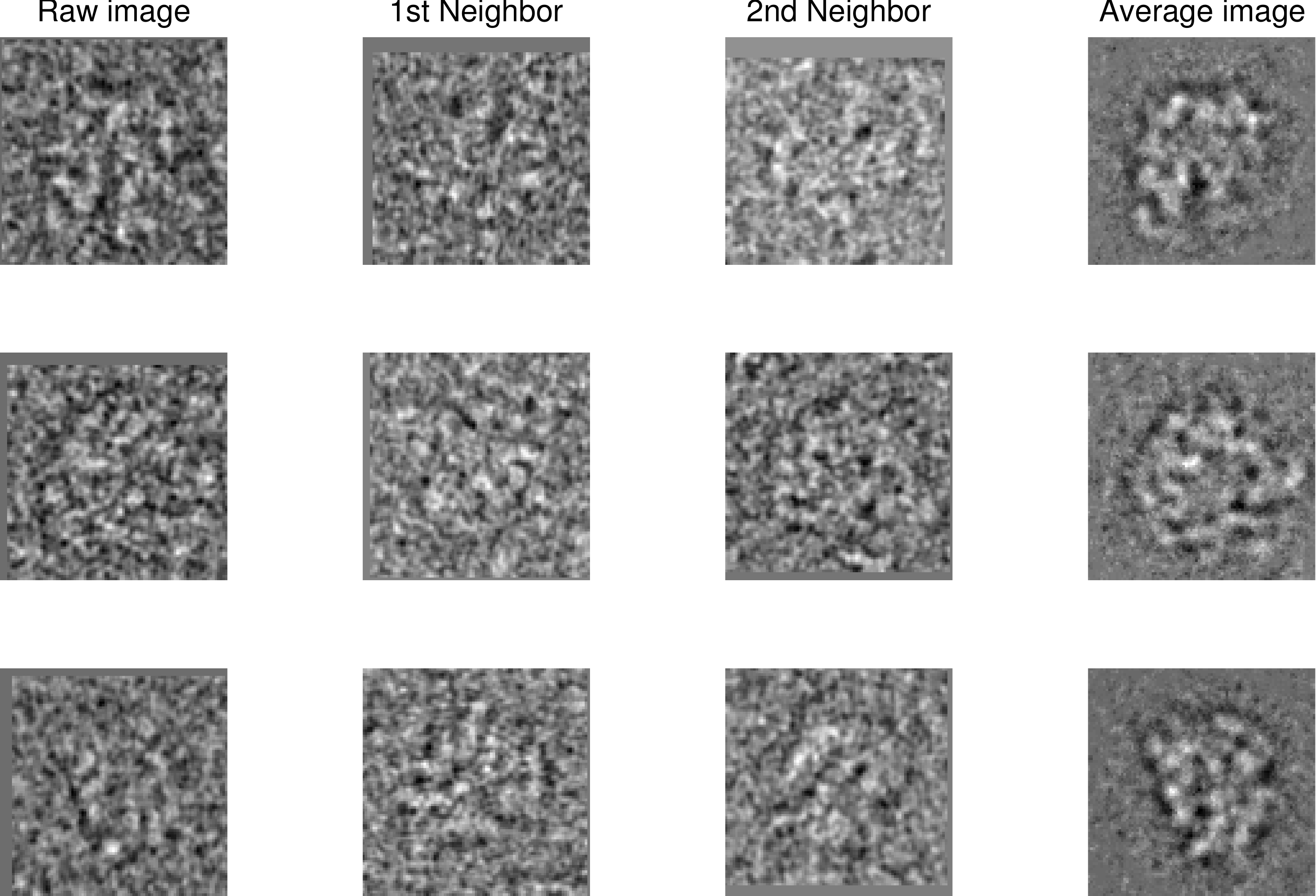}
\par\end{centering}
\caption{\label{fig:images_real}Noise reduction by image averaging. Three raw ribosomal images are shown in the first column. Their closest two neighbours (i.e., raw images having similar orientations after alignments) are shown in the second and third columns. The average images shown in the last column were obtained by averaging over $10$ neighbours of each raw image. }
\end{figure}
\begin{figure}
\centering
\subfloat[][\label{fig:initial_image}Initial models.]{
\includegraphics[width=0.3\paperwidth, angle = -90]{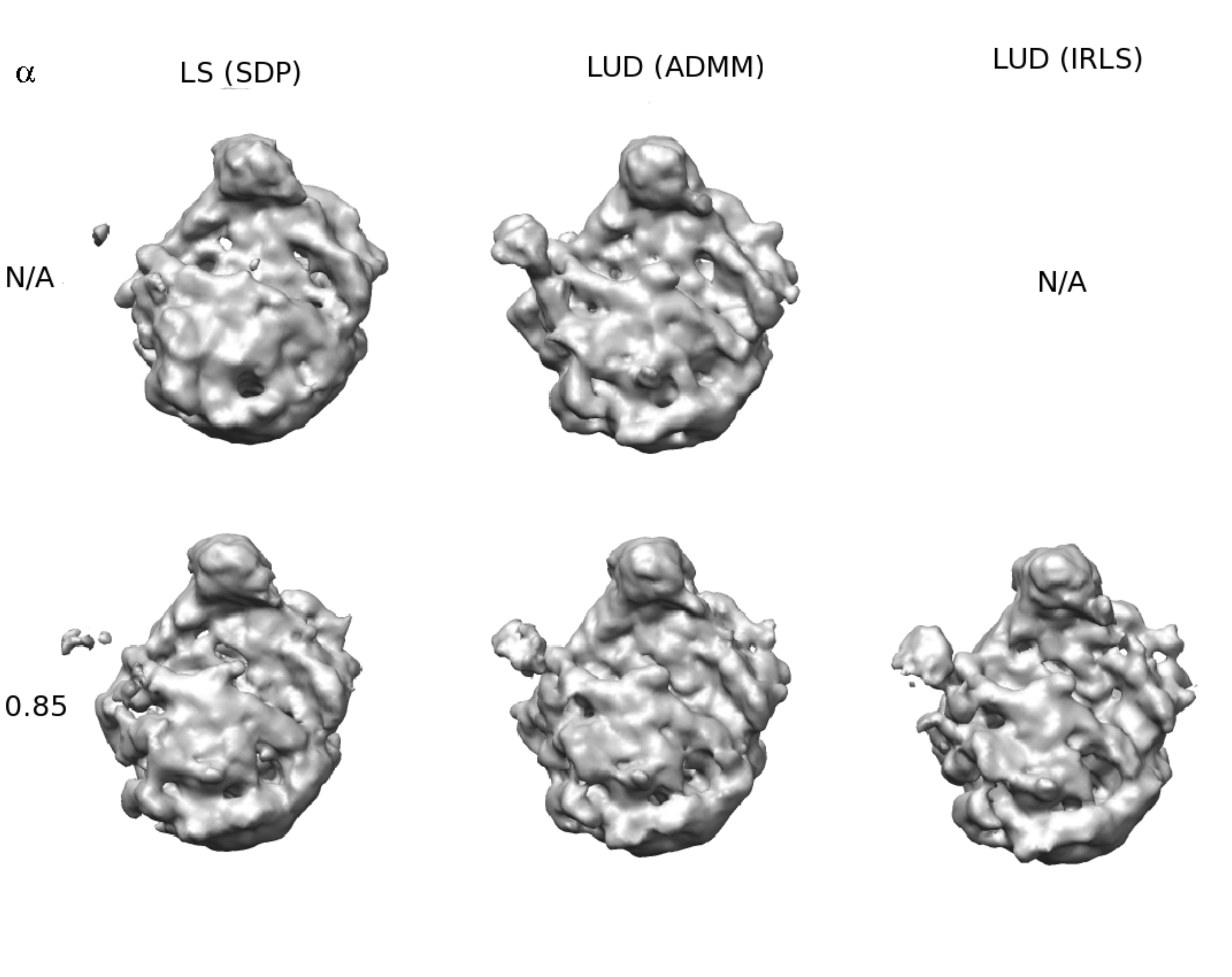}}\\
\subfloat[][\label{fig:final_image}Refined models.]{
\includegraphics[width=0.3\paperwidth, angle = -90]{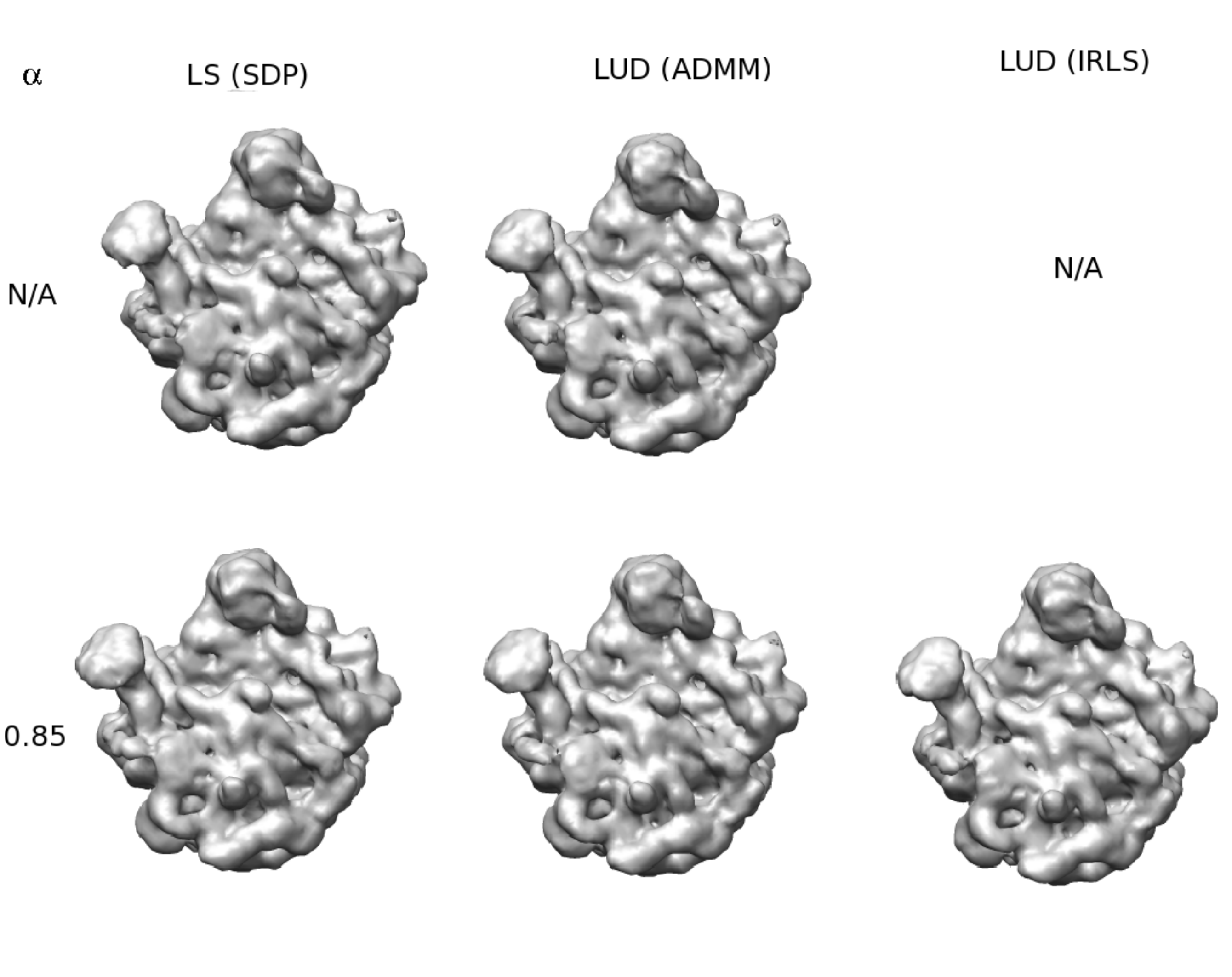}}
\caption{Initial models and refined models. (a)The ab-initio models estimated by merging two independent reconstructions, each obtained from $1000$ class averages. The resolutions of the models are 17.2\AA, 16.7\AA, 16.7\AA, 16.7\AA $\text{}$ and 16.1\AA $\text{}$ (from top to bottom, left to right) using the FSC $0.143$ resolution cutoff (Figure \ref{fig:fsc_real}). The model using the IRLS procedure without the spectral norm constraint  (i.e., $\alpha$ = N/A) is not available since the estimated projection directions are highly clustered. (b) The refined models corresponding to the ab-initio models in (a). The resolutions of the models are all 11.1\AA. }
\label{fig:reconstruction_real}
\end{figure}

\begin{figure}
  \centering
  \subfloat[][\label{fig:va}LS, SDP]{\includegraphics[width=.49\textwidth]{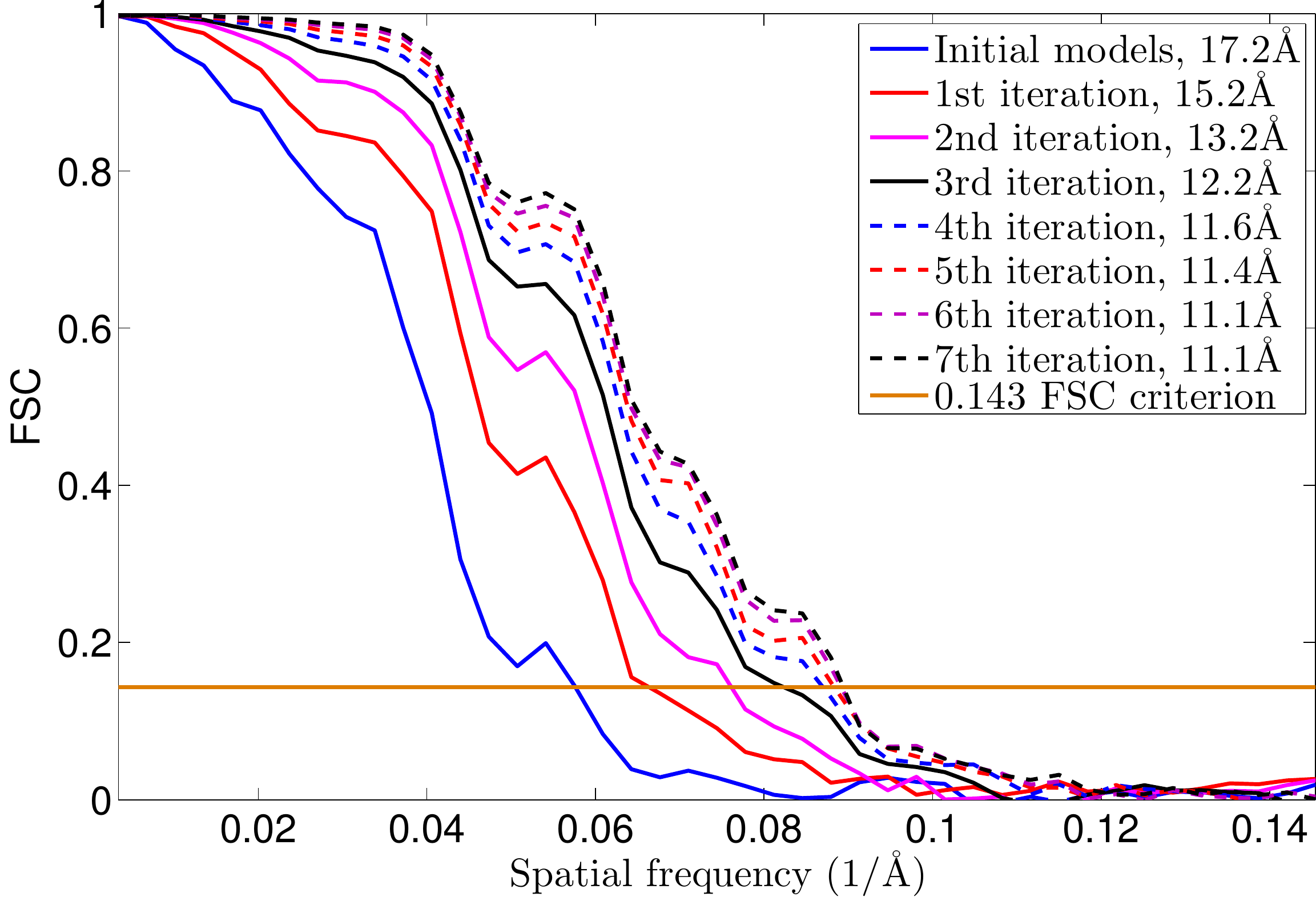}}\,
  \subfloat[][\label{fig:v}LUD, ADMM]{\includegraphics[width=.49\textwidth]{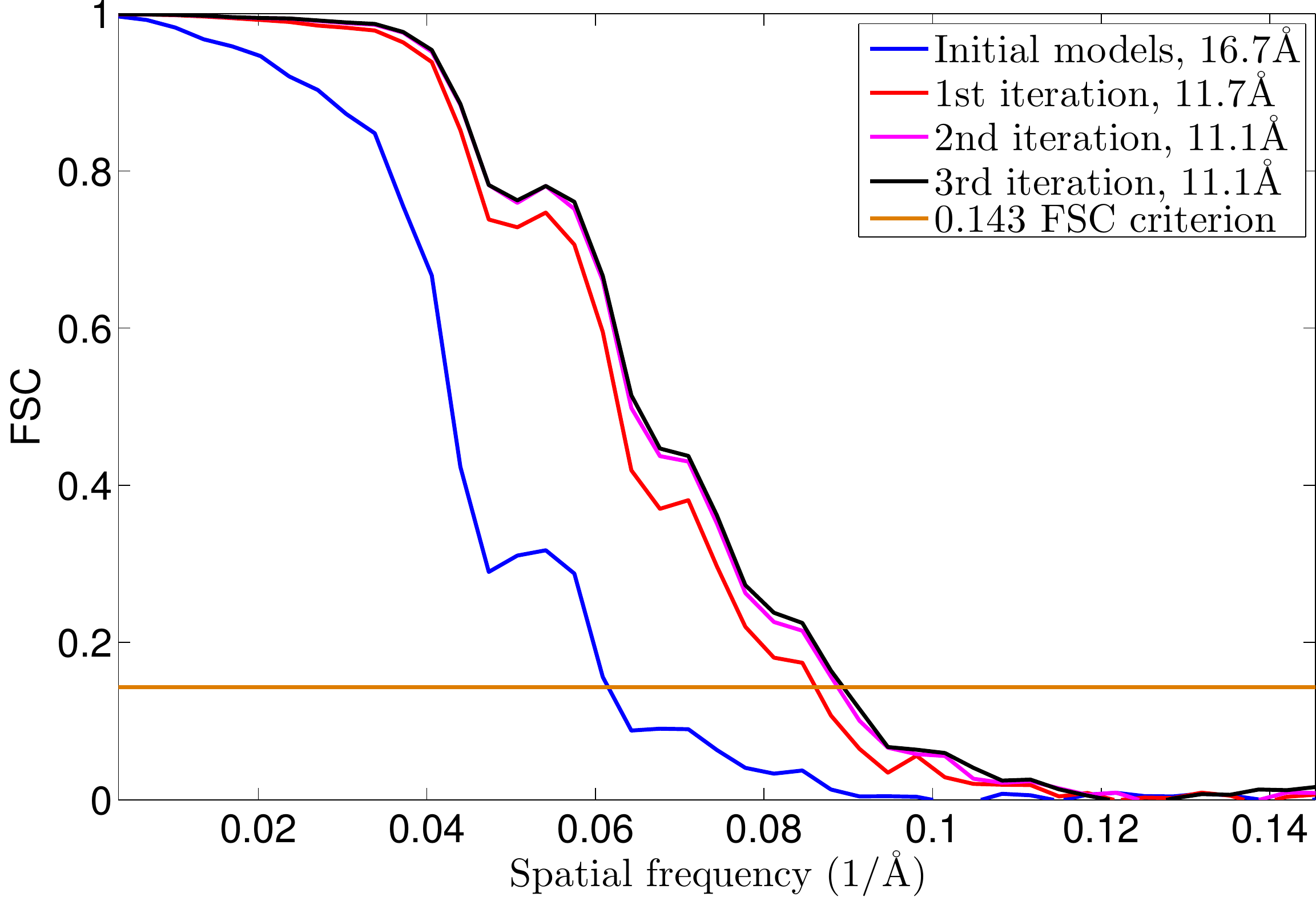}}\\
  \subfloat[][\label{fig:va_85}LS, ADMM, $\alpha = 0.85$]{\includegraphics[width=.49\textwidth]{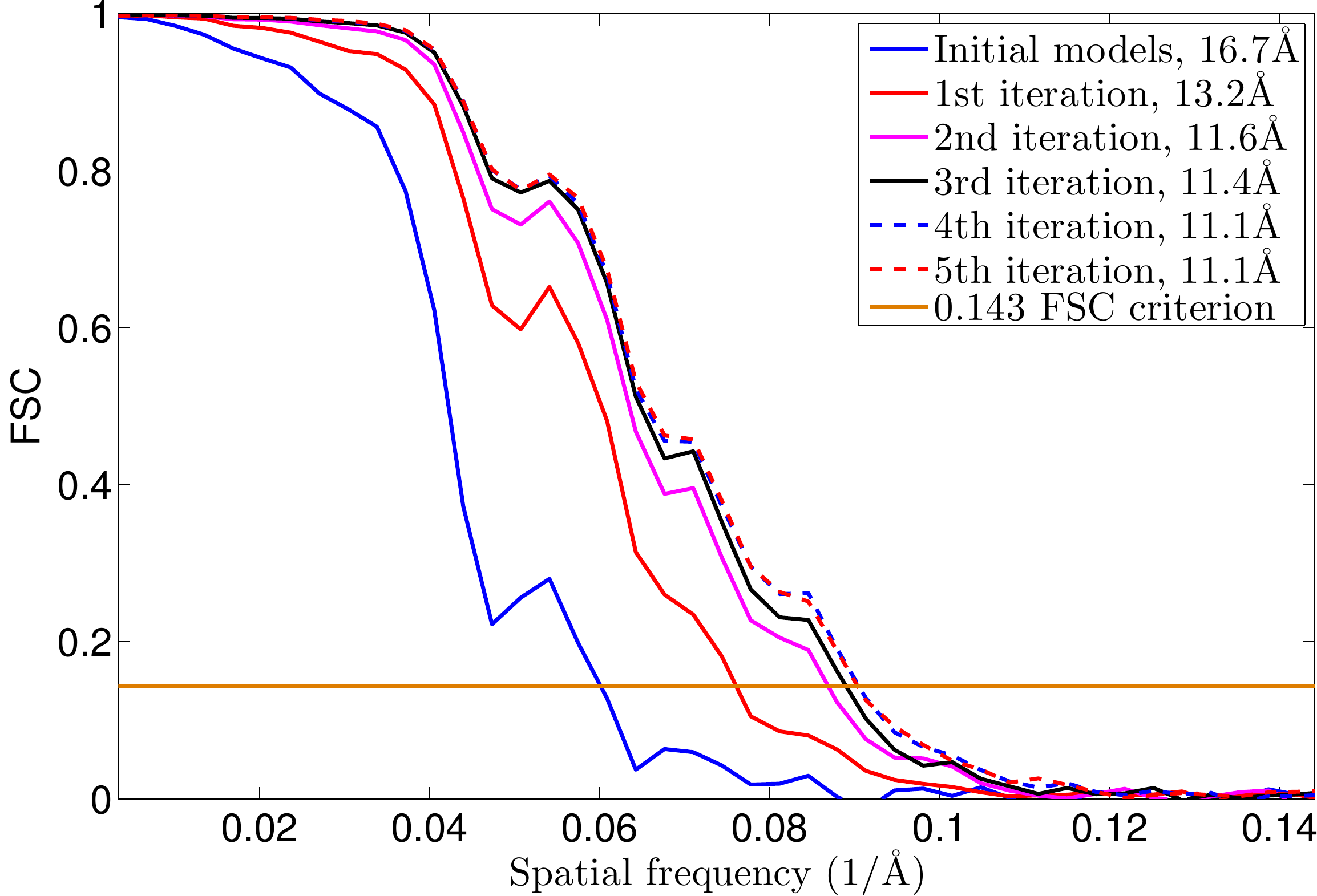}}\,
  \subfloat[][\label{fig:v_85}LUD, ADMM, $\alpha = 0.85$]{\includegraphics[width=.49\textwidth]{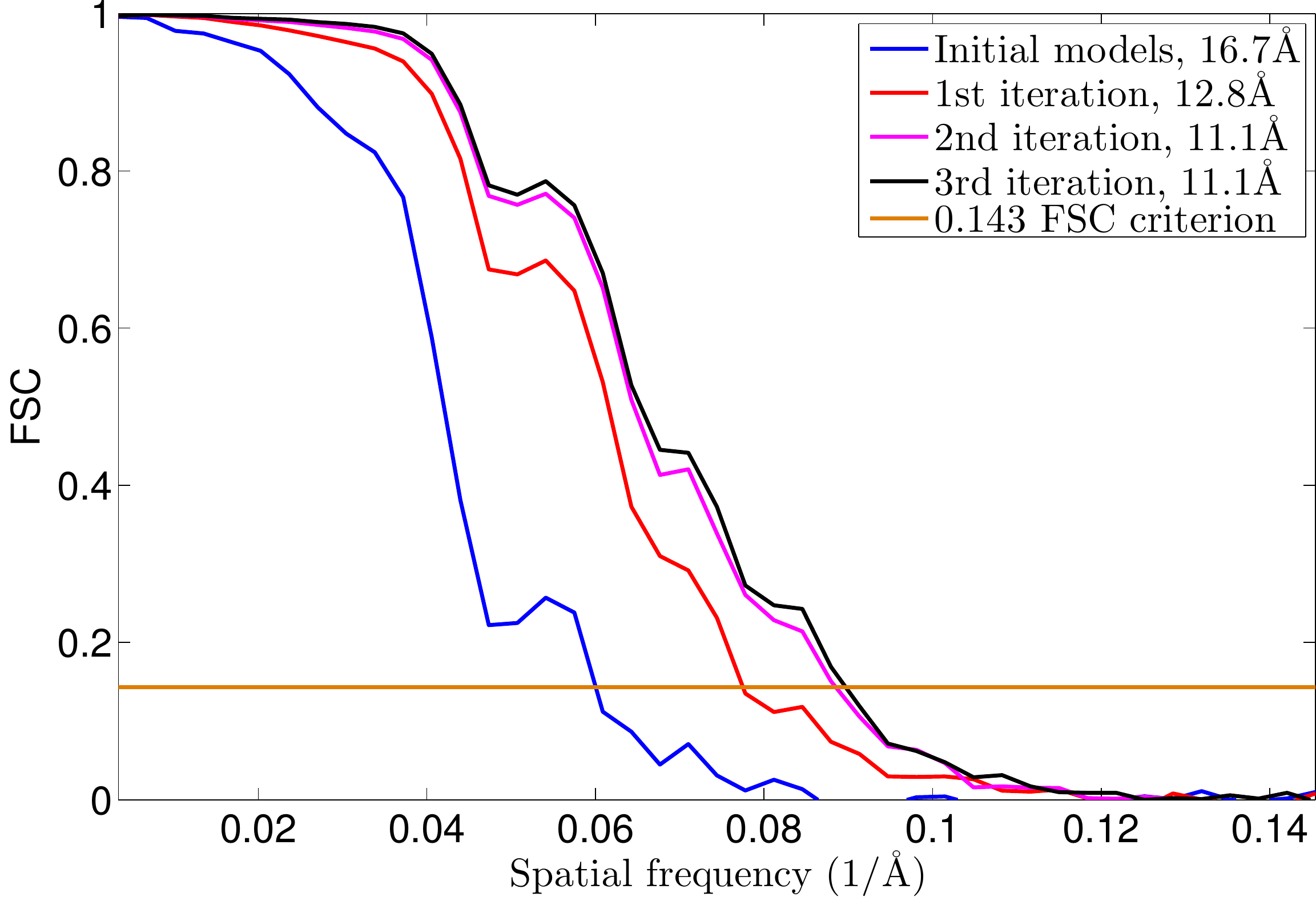}}\\
  \subfloat[][\label{fig:vb_85}LUD, IRLS, $\alpha = 0.85$]{\includegraphics[width=.49\textwidth]{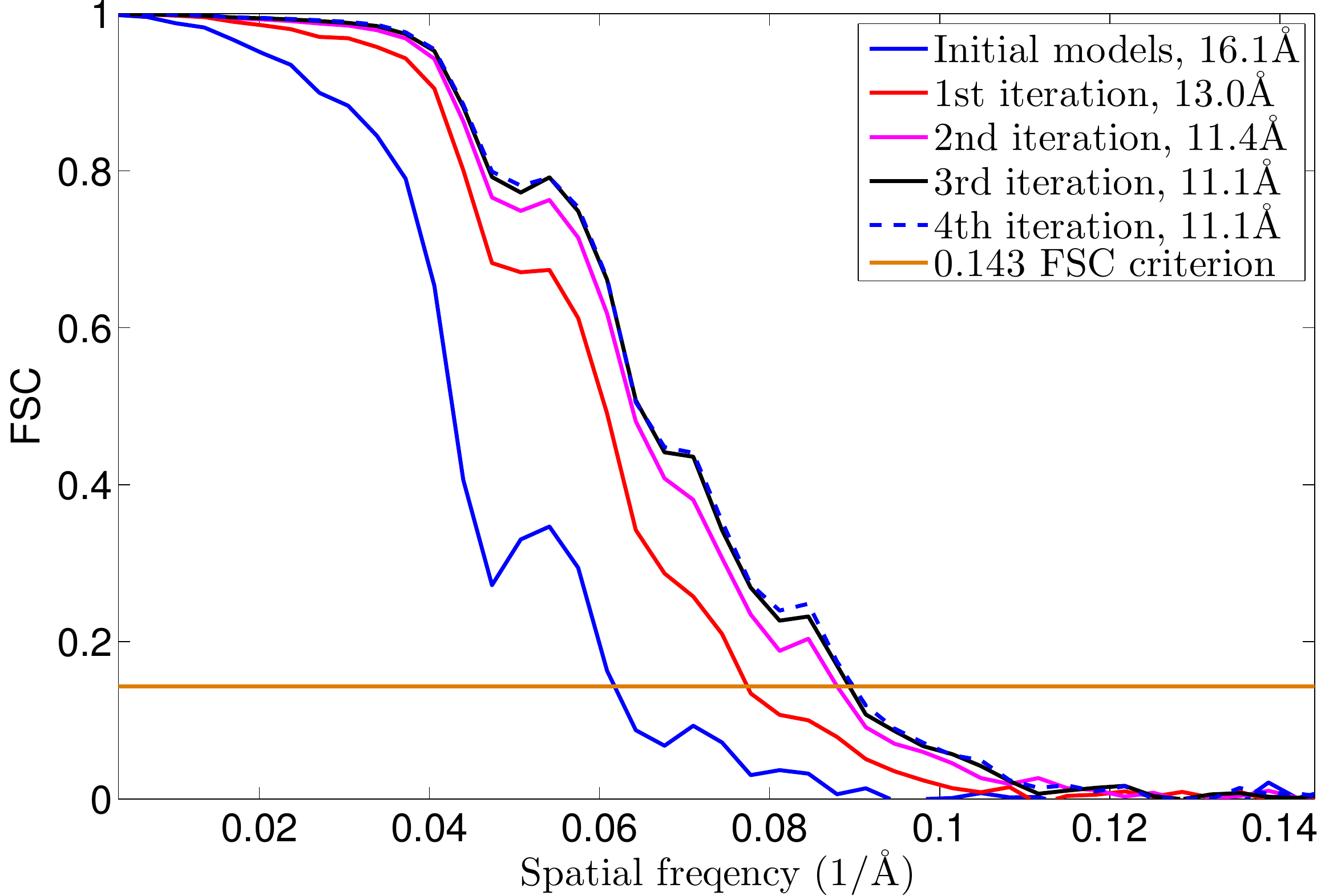}}\,
  \subfloat[][\label{fig:comparison}Comparison of refined models.]{\includegraphics[width=.49\textwidth]{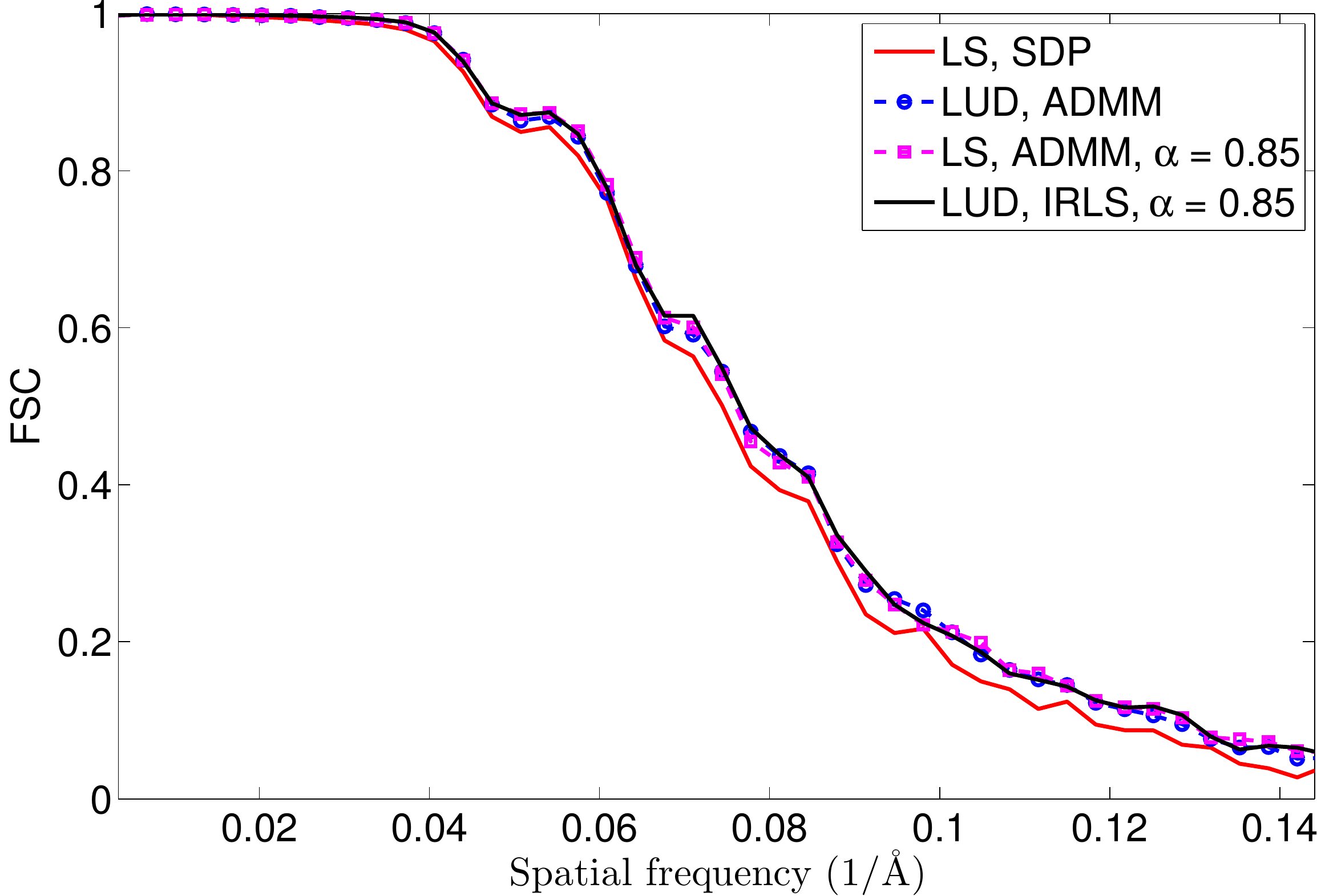}}
  \caption{Convergence of the refinement process. In sub-figure (a) - (e), the FSC plots show the convergence of the refinement iterations. The ab-initio models (Figure \ref{fig:initial_image} used in (a) - (e) were obtained by solving the LS/LUD problems using SDP/ADMM/IRLS. The numbers of refinement iterations performed in (a) - (e) are 7, 3, 5, 3 and 4 respectively. The sub-figure (f) are FSC plots of the refined models in (a), (b), (c) and (e) against the refined model in (d), which are measurements of similarities between the refined models in Figure \ref{fig:final_image}.}
  \label{fig:fsc_real}
\end{figure}
A set of micro-graphs of E. coli 50S ribosomal subunits was provided
by Dr. M. van Heel. These micro-graphs were acquired by a Philips
CM20 at defocus values between 1.37 and 2.06 $\mu m$, and they were
scanned at 3.36 \AA/pixel. The particles (particularly E. coli 50S ribosomal subunits) were picked using the automated
particle picking algorithm in EMAN Boxer \cite{EMAN}.  Then using the IMAGIC software package
(\cite{Stark2002, vanHeel199617}), the 27,121 particle images of size $90\times90$ pixels
were phase-flipped to remove the phase-reversals in the CTF,
bandpass filtered at 1/150 and 1/8.4 \AA, normalized by their variances, and
then translationally aligned with the rotationally-averaged total sum. The
particle images were randomly divided into 2 disjoint groups of equal number of images. The following steps were performed to each group separately. 

The images were rotationally aligned and averaged to produce “class averages” of better quality, following
the procedure detailed in \cite{zhao_class}. For each group, the images were
denoised and compressed using Fourier-Bessel based principal component analysis
(FBsPCA) \cite{fbs_pca}. Then, triple products of Fourier-Bessel expansion
coefficients obtained in FBsPCA were used to compute rotational invariant
features of the images, i.e, the bispectrums
\cite{Sadler:92,Kondor,Marabini:1996}. For each image, an initial set of neighboring images was
computed using the normalized cross-correlation of the bispectrums, which was
later refined using the method described in  \cite{Amit_classavg} to produce new sets of neighbors. 
Finally, for each image, we averaged it with its 10 nearest neighbors after
alignment. Three examples of averaged images are shown in Figure \ref{fig:images_real}.

One thousand class averages were randomly selected from each group. The LS  and
LUD approaches with and without the spectral norm constraint were applied. Two reconstructed volumes were obtained from
the two groups of images. The two resulting volumes were aligned and averaged to
obtain the ab-initio model (Figure \ref{fig:initial_image}). We observed that the
LUD approach gives much more reasonable ab initio models compared to the LS approach. In addition, the FSC of the two volumes was computed to
estimate the resolution of the ab-initio model (Figure \ref{fig:fsc_real}). Among all the ab-initio models, the one obtained by LS is at the lowest resolution 17.2\AA, while the one obtained by LUD through IRLS procedure is the highest resolution 16.1\AA.
Notice that the FSC measures the variance error, but not the bias error of the ab-initio
model. 
We also notice that the viewing directions of images estimated by the IRLS procedures without the spectral norm constraint converged to two clusters around two antipodal directions, resulting in no 3D reconstruction. Moreover, for this dataset, adding the spectral norm constraint on $G$ with $\alpha = 0.85$ did not improve the accuracy of the result, although this helped with regularizing the convergence in the IRLS procedure.

The two resulting volumes were then iteratively refined using 10,000 raw images in each group. In each refinement iteration, 2,000 template images were generated by projecting the 3D model from the previous iteration, then the orientations of the raw images were estimated using reference-template matching, and finally a new 3D model was reconstructed from the 10,000 raw images with highest correlation with the reference images. Each refinement iteration took about 4 hours. Therefore, a good ab-initio model should be able to accelerate the refinement process by reducing the total number of refinement iterations. The FSC plots in Figure \ref{fig:va} - \ref{fig:vb_85} show the convergence of the refinement process using different ab-initio models. We observed that all the refined models are at the resolution 11.1\AA. However, the worst ab-initio model obtained by LS needed $7$ iterations (about 28 hours) for convergence (Figure \ref{fig:va}), while the best ab-initio model obtained by LUD needed $3$ iterations (about 12 hours) for convergence (Figure \ref{fig:v} and Figure \ref{fig:v_85}). Figure \ref{fig:comparison} uses FSC plots to compare the refined models. We observed that the refined models in Figure \ref{fig:v} - \ref{fig:vb_85} were consistent to each other, while the refined model obtained by LS in Figure \ref{fig:va} was slightly different from others.

The average cost time for computing the ab-initio models in these two subsections are shown in
Table (\ref{tab:time}). It is not surprising to see that the LS approach was the fastest and that adding the spectral norm constraint slowed down the ADMM and IRLS procedures. The reason is that a large portion of the cost time in ADMM and IRLS is due to the projections onto the semidefinite cone. These steps are expected to be accelerated by the recent advance on eigenspace computation \cite{Zaiwen2013}. However, when using the LUD approach for the real data set, the time saved in the refinement was about  16 hours, which is much more than the time cost for computing the ab-initio models (about 0.5 - 1 hour when ADMM was used).

\begin{table}[h]
\center
\begin{tabular}{|c|c|c|c|c|c|c|}
\hline
 & \multicolumn{1}{c}{} & \multicolumn{1}{c}{$\alpha$ = N/A} & & \multicolumn{1}{c}{} & \multicolumn{1}{c}{$\frac{2}{3}\leq\alpha\leq1$} & \tabularnewline
\cline{2-7}
$K$ & LS & \multicolumn{1}{c}{LUD} &  & LS & \multicolumn{1}{c}{LUD} & \tabularnewline
\cline{3-4} \cline{6-7}
 & (SDP) & ADMM  & IRLS & (ADMM) & ADMM & IRLS\tabularnewline
\hline
$500$ & 7s & 266s & 469s & 78s & 454s & 3353s\tabularnewline
\hline
$1000$ & 31s & 1864s & 3913s & 619s & 1928s & 20918s\tabularnewline
\hline
\end{tabular}
\caption{\label{tab:time}The average cost time using different algorithms on $500$ and $1000$ images in the two experimental subsections. The notation $\alpha$ = N/A means no spectral norm constraint $\left\Vert G \right\Vert _2 \leq \alpha K$ is used.}
\end{table}

\section{Discussion}
\label{sec:discussion}

To estimate image orientations, we introduced a robust self consistency error and used ADMM or the IRLS procedure to solve the associated LUD problem after SDR. Numerical experiments demonstrate that the solution is less sensitive to outliers in  the detected common-lines than the LS method approach. In addition, when the common-line detection rate is low, the spectral norm constraint on the Gram matrix $G$ can help to tighten the semidefinite relaxation, and thus improves the accuracy of the estimated rotations in some cases. Moreover, the numerical experiments using the real data set (Section \ref{sec:real_data}) demonstrate that the ab-initio models resulted by the LUD based methods are more accurate than initial models that are resulted by least squares based methods. In particular, our initial models requires
fewer time-consuming refinement iterations. We note that it is also possible to consider other self consistency errors involving the unsquared deviations raised to some power $p$ (e.g., the cases $p=1, 2$ correspond to LUD and LS, respectively).  We observed that the accuracy of the estimated orientations can be improved by using $p<1$ provided that the initial guess is ``sufficiently good".
The LUD approach and the spectral norm constraint on $G$ can be generalized to the synchronization approach to estimate the images' orientations in \cite{sync_cryoem}.

In  \cite{LUD}, the LUD approach is shown to be more robust than the LS approach for the synchronization problem over the rotation group $SO(d)$. Given some relative rotations $R_i ^TR_j$, the synchronization problem is to estimate the rotations $R_i\in SO(d)$, $i = 1, \ldots, K$ up to a global rotation. It is verified that under a specific model of the measurement noise and the measurement graph for $R_i^TR_j$, the rotations can be exactly and stably recovered using LUD, exhibiting a phase transition behavior in terms of the proportion of noisy measurements. The problem of orientation determination using common-lines between cryo-EM images is similar to the synchronization problem. The difference is that the pairwise information given by the relative rotation $R_i^T R_j$ is full, while that given by the common-lines $\vec{c}_{ji}^T \vec{c}_{ij}$ is partial. Moreover, the measurement noise of each detected common-line $\vec{c}_{ij}$ depends on image $i$ and $j$, and thus it cannot be simply modeled, which brings the difficulties in verifying the conditions for the exact and stable orientation determination we observed.

\section{Acknowledgements}
The authors would like to thank Zhizhen Zhao for producing class averages from the experimental ribosomal
images. The work of L. Wang and A. Singer was partially supported by Award
Number FA9550-12-1-0317 from AFOSR, by Award Number R01GM090200 from the NIGMS,
by the Alfred P. Sloan Foundation, and by the Simons Foundation. The work of Z.
Wen was partially supported by NSFC grant 11101274.

\bibliographystyle{plain}
\bibliography{LUD_cryoem}

\Appendix
\section{Exact recovery of the Gram matrix $G$ from correct common-lines}
\label{sec:exact_G}
Here we prove that if the detected common-lines $\vec{c}_{ji}$ (defined in (\ref{eq:common-line-notation})) are all correct and  at least three images have linearly independent projection directions (i.e., the viewing directions of the three images are not on the same great circle on the sphere shown in Figure \ref{fig:alpha}), then the Gram matrix $G$ obtained by solving the LS problem (\ref{eq:sdp})-(\ref{eq:sdp_constraint}) or the LUD problem (\ref{eq:lad-exact_sdr_2}) is uniquely the one defined in (\ref{eq:G}). To verify the uniqueness of the solution $G$, it is enough to show rank($G$)$= 3$ due to the SDP solution uniqueness theorem (page 36-39 in \cite{aspects_of_sdp}, \cite{Zhu2010}). Without loss of generality, we consider the SDP for the LS approach when applied on three images (i.e., $K = 3$ and $w_{ij} = 1$ in the problem (\ref{eq:sdp}) - (\ref{eq:sdp_constraint})):
\[
\max_{G_{6\times6}\succcurlyeq0}\sum_{i,j=1,2,3}\left\langle G_{ij},\vec{c}_{ji}^{T}\vec{c}_{ij}\right\rangle \text{ s.t. }G_{ii}=I_{2},
\]
Since the solution $G$ is positive semidefinite, we can decompose $G$ as
\[
G=\left(\begin{array}{c}
\mathbf{u}_{1}^{1^{T}}\\
\mathbf{u}_{1}^{2^{T}}\\
\mathbf{u}_{2}^{1^{T}}\\
\mathbf{u}_{2}^{2^{T}}\\
\mathbf{u}_{3}^{1^{T}}\\
\mathbf{u}_{3}^{2^{T}}
\end{array}\right)\left(\begin{array}{cccccc}
\mathbf{u}_{1}^{1} & \mathbf{u}_{1}^{2} & \mathbf{u}_{2}^{1} & \mathbf{u}_{2}^{2} & \mathbf{u}_{3}^{1} & \mathbf{u}_{3}^{2}\end{array}\right),
\]
where ${\bf u}^p_i$, $p = 1, 2$ and $i = 1, 2, 3$ are column vectors.
We will show rank($G$)$=3$, i.e., any four vectors among $\left\{ \mathbf{u}_{1}^{1},\mathbf{u}_{1}^{2},\mathbf{u}_{2}^{1},\mathbf{u}_{2}^{2},\mathbf{u}_{3}^{1},\mathbf{u}_{3}^{2}\right\} $
span a space with dimensionality at most $3$.

Define arrays ${\bf u}_i$ as
\[
\mathbf{u}_{i}=(\mathbf{u}_{i}^{1},\mathbf{u}_{i}^{2}),
\]
then the inner product
\begin{eqnarray*}
\left\langle G_{ij},\vec{c}_{ji}^{T}\vec{c}_{ij}\right\rangle  & = & \left\langle \mathbf{u}_{i}^{T}\mathbf{u}_{j},\vec{c}_{ji}^{T}\vec{c}_{ij}\right\rangle \\
 & = & \left\langle \vec{c}_{ji}\mathbf{u}_{i},\vec{c}_{ij}\mathbf{u}_{j}\right\rangle \\
 & = & \left\langle c_{ji}^{1}\mathbf{u}_{i}^{1}+c_{ji}^{2}\mathbf{u}_{i}^{2},c_{ij}^{1}\mathbf{u}_{j}^{1}+c_{ij}^{2}\mathbf{u}_{j}^{2}\right\rangle \\
 & \leq & 1,
\end{eqnarray*}
where the last inequality follows the Cauchy-Schwarz inequality and the facts that all $\vec{c}_{ij}$
are unit vectors, $\mathbf{u}_{i}^{1}$ and $\mathbf{u}_{i}^{2}$
are unit vectors and orthogonal to each other due to the constraint
$G_{ii}=I_{2}$, and thus all $c_{ij}^{1}\mathbf{u}_{j}^{1}+c_{ij}^{2}\mathbf{u}_{j}^{2}$
are unit vectors on the Fourier slices of the images. The equality holds if and only if
\begin{equation}
c_{ji}^{1}\mathbf{u}_{i}^{1}+c_{ji}^{2}\mathbf{u}_{i}^{2}=c_{ij}^{1}\mathbf{u}_{j}^{1}+c_{ij}^{2}\mathbf{u}_{j}^{2}.\label{eq:1}
\end{equation}
Thus when the maximum is achieved, due to (\ref{eq:1}) and the fact that the projection directions of the images are linearly independent, dim(span\{$\mathbf{u}_{i}^{1},\mathbf{u}_{i}^{2}$\}$\cap$span\{$\mathbf{u}_{j}^{1},\mathbf{u}_{j}^{2}$\})$=1$
and thus dim(span\{$\mathbf{u}_{i}^{1},\mathbf{u}_{i}^{2},\mathbf{u}_{j}^{1},\mathbf{u}_{j}^{2}$\})$=3$.
Therefore, without loss of generality, we only have to show that dim(span\{$\mathbf{u}_{1}^{1},\mathbf{u}_{2}^{1},\mathbf{u}_{3}^{1},\mathbf{u}_{3}^{2}$\})$\leq3$.
Using (\ref{eq:1}), assume that span\{$\mathbf{u}_{1}^{1},\mathbf{u}_{1}^{2}$\}$\cap$span\{$\mathbf{u}_{3}^{1},\mathbf{u}_{3}^{2}$\}=span\{$\mathbf{v}_{1}$\}
and span\{$\mathbf{u}_{2}^{1},\mathbf{u}_{2}^{2}$\}$ \cap $span\{$\mathbf{u}_{3}^{1},\mathbf{u}_{3}^{2}$\} = span\{$\mathbf{v}_{2}$\},
 where $\mathbf{v}_{1}$ and $\mathbf{v}_{2}$ are linearly independent vectors
(otherwise all three projection directions are linearly dependent and thus the 3 Fourier slices of the images intersect at the same line). Therefore we
have span\{$\mathbf{v}_{1},\mathbf{v}_{2}$\}=span\{$\mathbf{u}_{3}^{1},\mathbf{u}_{3}^{2}$\},
span\{$\mathbf{v}_{1},\mathbf{u}_{1}^{1}$\}$\subseteq$span\{$\mathbf{u}_{1}^{1},\mathbf{u}_{1}^{2}$\}
and span\{$\mathbf{v}_{2},\mathbf{u}_{2}^{1}$\}$\subseteq$span\{$\mathbf{u}_{2}^{1},\mathbf{u}_{2}^{2}$\}.
Thus dim(span\{$\mathbf{u}_{1}^{1},\mathbf{u}_{2}^{1},\mathbf{u}_{3}^{1},\mathbf{u}_{3}^{2}$\})
= dim(span\{$\mathbf{u}_{1}^{1},\mathbf{u}_{2}^{1},$ $\mathbf{v}_{1},\mathbf{v}_{2}$\})
$\leq$ dim(span\{$\mathbf{u}_{1}^{1},\mathbf{u}_{1}^{2},\mathbf{u}_{2}^{1},\mathbf{u}_{2}^{2}$\})
$\leq3$.


\end{document}